\documentclass{article}
\usepackage{iclr2027_conference,times}

\iclrfinalcopy

\usepackage{amsmath,amssymb}
\usepackage{thmtools}
\usepackage{amsthm}
\usepackage{mathtools}
\usepackage{thm-restate}
\usepackage{xcolor}
\usepackage{enumitem}
\usepackage{wrapfig}
\usepackage{comment}
\usepackage{algorithm}
\usepackage{graphicx}
\usepackage{wrapfig}
\usepackage{booktabs}
\usepackage{caption}
\usepackage{tabularx}
\usepackage{float}
\usepackage{bm}

\makeatletter
\newcommand{\stopwrapping}{%
  \par
  \ifx\parshape\WF@fudgeparshape
    \nobreak
    \ifnum\c@WF@wrappedlines>\@ne
      \advance\c@WF@wrappedlines\m@ne
      \vskip\c@WF@wrappedlines\baselineskip
      \global\c@WF@wrappedlines\z@
    \fi
    \allowbreak
    \WF@finale
  \fi
  \WFclear
}
\makeatother

\usepackage{algpseudocode}
\usepackage[
colorlinks=true,
linkcolor=blue!60!black,
citecolor=green!40!black,
urlcolor=blue!60!black]{hyperref}
\usepackage[noabbrev,capitalise,nameinlink]{cleveref}
\newtheorem{theorem}{Theorem}
\newtheorem{lemma}{Lemma}

\newtheorem{corollary}{Corollary}
\newcommand{\ip}[2]{\langle #1,#2\rangle}
\newcommand{\eps}{\varepsilon}
\newcommand{\E}{\mathbb E}
\newcommand{\R}{\mathbb R}

\newcommand{\OPTattn}{\mathbf{OPT}_{\mathsf{Attn}}}
\newcommand{\Attn}{\operatorname{Attn}}
\DeclareMathOperator{\Var}{Var}
\newcounter{todocounter}



\newtheorem{definition}[theorem]{Definition}

\setlist{nosep,leftmargin=*}

\definecolor{mainterm}{gray}{0}
\definecolor{secondary}{gray}{0.5}
\newcommand{\mainterm}[1]{{\color{mainterm}\bm{#1}}}
\newcommand{\secondary}[1]{{\color{secondary}#1}}
\title{Attention under Bounded Key–Query Similarity: Space Complexity and Transformer Anisotropy}

\author{
Justin Y. Chen, Ying Feng \& Piotr Indyk \\ Massachusetts Institute of Technology (MIT). \\
\texttt{\{justc,yingfeng,indyk\}@mit.edu} \\
\AND
Michael Kapralov, Ekaterina Kochetkova \& Boris Prokhorov \\
École Polytechnique Fédérale de Lausanne (EPFL). \\
\texttt{\{firstname.lastname\}@epfl.ch}
}

\date{}
\begin{document}
\maketitle
\fancyhead{}
% \vspace{-2.6em}

\begin{abstract}
% must be one paragraph
A growing body of work on KV-cache compression aims to reduce the memory requirements of autoregressive generation. These methods compress the stored context while approximately preserving attention computation, enabling generation with reduced memory consumption. Almost tight theoretical results on the worst case space complexity of attention have recently been obtained, but they are achieved by instances where queries can be perfectly aligned with keys. This directly contradicts the observed anisotropy of the key/query geometry in real transformers. In particular, it has been shown empirically that key/query cosine similarities are small even for most correlated pairs, i.e., near-perfect alignment of queries with keys does not happen in practice. In this paper, we study the space complexity of attention approximation under the assumption that key/query cosine similarities belong to a small interval around zero, i.e., a query cannot align too much with {\em any} key. We give a nearly space optimal algorithm and a corresponding space lower bound. Surprisingly, the distribution of key/query cosine similarities in our lower bound exhibits the {\em anisotropy} property empirically established and studied in the literature. Additionally, we show that the simple  {\em uniform sampling} algorithm is nearly optimal for a typical range of parameters, supporting its surprising effectiveness on real datasets.
\end{abstract}

%As per Ying, should perhaps cite~\cite{wang2026randomattention}\eee{we did, so remove? or do we want to cite in the abstract?} \boris{I think remove, we decided to not cite anyone here}

\section{Introduction}

  The high memory consumption of modern large language models (LLMs) is a major challenge in their deployment, as the decoding stage of transformers requires storage of all tokens seen so far. The growing body of work on Key-Value (KV) cache compression has been proposed to address this challenge. %These algorithms compress the tokens so that generation can be done with reduce memory consumption. 
    These algorithms rely on innovative ways of approximating the attention function of the context using small space, either by {\em quantizing} the vectors, or {\em evicting} some of them to retain the most informative ones. In both cases, the goal is to approximate the value of the attention function:
\begin{equation}\label{eq:attention-definition}
    \Attn(K,V,q)
    =\frac{\sum_{i=1}^n e^{\ip{k_i}{q}/\sqrt{d}}v_i}
           {\sum_{i=1}^n e^{\ip{k_i}{q}/\sqrt{d}}}.
\end{equation}
where $K=(k_1,\ldots,k_n)\subset\R^d$ is the set of all {\em keys},
$V=(v_1,\ldots,v_n)\subset\R^d$ is the set of corresponding {\em values} and query $q\in\R^d$ is the {\em query}. 

We consider randomized algorithms for token eviction that are \emph{query-oblivious}.
That is, before seeing the specific query, the algorithm constructs a weighted summary by selecting a random 
subset of token indices $I\subseteq[n]$ together with positive weights $w_i$ (and evicts the other tokens). Given
a query $q$, the algorithm returns
\begin{equation}\label{eq:attention-model}
    \widehat{\Attn}(K,V,q)
    =\frac{\sum_{i\in I}w_i e^{\ip{k_i}{q}/\sqrt{d}}v_i}
           {\sum_{i\in I}w_i e^{\ip{k_i}{q}/\sqrt{d}}}.
\end{equation}
For every fixed query $q$, we require
\begin{equation}\label{eq:attention-guarantee}
    \Pr\!\left[
        \bigl\|\widehat{\Attn}(K,V,q)-\Attn(K,V,q)\bigr\|
        \le\eps
    \right]\ge\frac23.
\end{equation}
The probability is over the randomness used to construct the
summary. The query is fixed independently of this randomness; we
do not require simultaneous success for all queries. Assume that $\|k_i\|\le r_k$ for every $i$ and $\|q\|\le r_q$. We will assume $\|v_i\|\le1$. 

  % \ying{Could someone check how $\sqrt{d}$ carries over in the following paragraphs? I think some $\exp(r_k r_q)$ now becomes $\exp(r_k r_q/\sqrt{d})$ but I'm not 100\% sure.} 
  % Boris: I believe it is fixed
  
  Query-oblivious weighted token eviction (a.k.a.\ coresets) has been a subject of multiple recent theoretical investigations, leading to (almost) tight bounds on the {\em worst case} space complexity ~\cite{HarisO25,libertyak26,chen2026tight}, see \Cref{tab:related_work} for details. 
  % \piotr{Table in Related Work?}
  These results show that an $e^{r_k r_q/\sqrt{d}}$ dependence in the space bound is sufficient and necessary. While these results are close to optimal {\em in the worst case}, some of their implications significantly deviate from what one sees in practical instances:
  \begin{enumerate}
      \item  The space lower bound of $e^{r_k r_q/\sqrt{d}}$  can be much larger than the size of the input itself, given that the values of $r_{k},r_{q}$ are often on the order of $15-20$ (see, for example, Table~3 in~\cite{kochetkova2025balancekv} and our \Cref{fig:norms}) and $d = 128$; in such cases, the upper bound becomes vacuous.
      \item The hard distribution of the keys used to prove the lower bounds is isotropic, whereas in practice, keys and queries are known to be {\em anisotropic}, concentrating in a relatively narrow cone of the high-dimensional space~\cite{GodeyCS24}.
      \item The worst case hard instances have the queries align perfectly (i.e., have a high dot product) with one of the keys; in contrast, in practice, the key-query cosine similarity is small in absolute value once the attention sink is excluded and stored separately (see e.g., \cite{mazare2025saap} and our \Cref{fig:cos_sim_main}).
  \end{enumerate}

\subsection{Our results}
Reflecting this discrepancy, we study the space complexity of attention approximation under the following assumption:
    \begin{center}
    {\bf ($\star$)}  Every key/query cosine similarity  lies in some range $[a, b]$, where the magnitudes $|a|, |b| \ll 1$, i.e., no query can align significantly with {\em any} key.    
    \end{center}

Before we state our main theorem, we need to introduce some notation. Fix radius parameters $r_k, r_q$.
Suppose $\|k_i\|\le r_k$ for every $i$, and fix a pair of parameters
$a, b \in [-1,1]$. We take the admissible query set to be\footnote{Throughout this paper, we assume that $\|k_i\|, \|q\|$ are non-zero.}
\begin{align}
\begin{split}\label{eq:set}
    Q:=\Bigl\{q\in\R^d\!: \,
       \|q\|\le r_q,\ a \leq \frac{\ip{k_i}{q}}{\|k_i\|\|q\|}\le b \quad
     \forall i\in[n]\Bigr\}.
\end{split}
\end{align}

\begin{wrapfigure}{l}{0.35\textwidth}    \centering
    \begin{minipage}[]{\linewidth}
        \vspace{-5mm}
        \includegraphics[width=\linewidth]{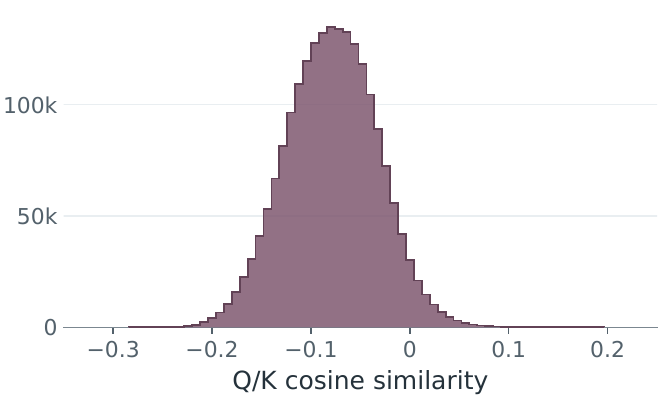}
    \end{minipage}
    \caption{Cosine similarities between the final 256 queries and preceding context keys in layer 8 of Qwen3-8B, for an 8K-token LongBench v2 input, excluding the first key (attention sink).}
    \label{fig:cos_sim_main}
    \vspace{-5mm}
\end{wrapfigure}

Let $\OPTattn$ denote the optimal worst-case number of retained
tokens subject to
\[
 \Pr\!\left[
 \left\|
 \widehat{\operatorname{Attn}}(K,V,q)
 -
 \operatorname{Attn}(K,V,q)
 \right\|\le\eps
 \right]\ge\frac23
\]
for every fixed admissible $q$. 
The algorithm knows keys, values, and parameters $r_k, r_q, a, b$, but
not the particular query in $Q$.
All norms are Euclidean, and all logarithms are natural. We use
$\widetilde O$ and $\widetilde\Omega$ to suppress factors
polylogarithmic in $n$.

% \stopwrapping

Our main result shows that importance sampling is a nearly optimal token eviction policy for attention approximation. Surprisingly, for a wide range of parameters (see \Cref{fig:importance-sampling}) all importance weights can be set to $1$, i.e. the most basic uniform sampling is optimal.  Formally, we define two regions in the space of parameters $(a, b)$, see also \Cref{fig:importance-sampling} for an illustration:

\begin{definition}[Uniform \& importance sampling regimes]\label{def:small_cos_sim}
We say that fixed parameters $-1\le a\le b\le1$ lie in the
\emph{uniform sampling} regime if $a<0<b$ and $b - a \leq \frac{1+b}{2}$.
% \Leftrightarrow b-2a\le1$.
% \boris{should we change to $b - a < \frac{1 + b}{2}$? Seems more intuitive given our theorem.} \eee{edited like this; for figures $b - 2a \leq 1$ is better}

We say that fixed parameters $-1\le a\le b\le1$ lie in the
\emph{importance sampling} regime in any other case.
\end{definition}

Our main result is as follows.

\begin{theorem}
\label{thm:attention}
Fix any constant $\eps\in(0,1/2]$ independently of $n$, and assume
$d=\omega(\log n)$. Fix parameters $r_k,r_q>0$ and $-1\le a<b\le1$. Suppose $\|k\| \leq r_k$ for every key, and the admissible query set $Q$ is described as \cref{eq:set}.
Let $R^2=r_kr_q/\sqrt d$, then

\begin{enumerate}
    \item\label{item:uniform} In the \emph{uniform sampling} regime, uniform sampling
    achieves asymptotically tight space bounds:
    \[
    \secondary{
        \mainterm{\Omega\Bigl(}
            {\min\left\{n,\, \mainterm{e^{(b-a)R^2}}\right\}}
            \cdot e^{-o(R^2)}
        \mainterm{\Bigr)}
    }
    \ \le\ \OPTattn\ \le\
    \secondary{
        \mainterm{O\Bigl(}
            {\min\left\{n,\, \mainterm{e^{(b-a)R^2}}\right\}}
        \mainterm{\Bigr)}
    }.
    \]

    \item\label{item:importance} In the \emph{importance sampling} regime, importance sampling
    achieves asymptotically tight space bounds:
    \[
    \secondary{
        \mainterm{\Omega\Bigl(}
            {\min\left\{n,\, \mainterm{e^{\frac{1+b}{2}R^2},\, e^{(b-a)R^2}}\right\}}
            \cdot e^{-o(R^2)}
        \mainterm{\Bigr)}
    }
    \ \le\ \OPTattn\ \le\
    \secondary{
        \mainterm{O\Bigl(}
            {\min\left\{n,\, \mainterm{e^{\frac{1+b}{2}R^2},\, e^{(b-a)R^2}}\right\}}
        \mainterm{\Bigr)}
    }.
    \]
\end{enumerate}

% Uniform sampling attains the upper bound in the small cosine similarities regime. Importance sampling attains the upper bound for every regime.
% The $o(R^2)$ term in the lower bound  is $O(R^2\cdot\sqrt{\log n/d}) = O\bigl(\frac{r_k r_q \sqrt{\log n}}{d}\bigr) \approx 10$ in practice.

% Boris: oops, o(1) term is big... But this is ok, since our bound in this case is also big.
\end{theorem}

\begin{figure}[htbp]
    \centering
    \begin{minipage}[t]{0.4\textwidth}
        \vspace{2pt}
        \centering
        \includegraphics[width=\linewidth]{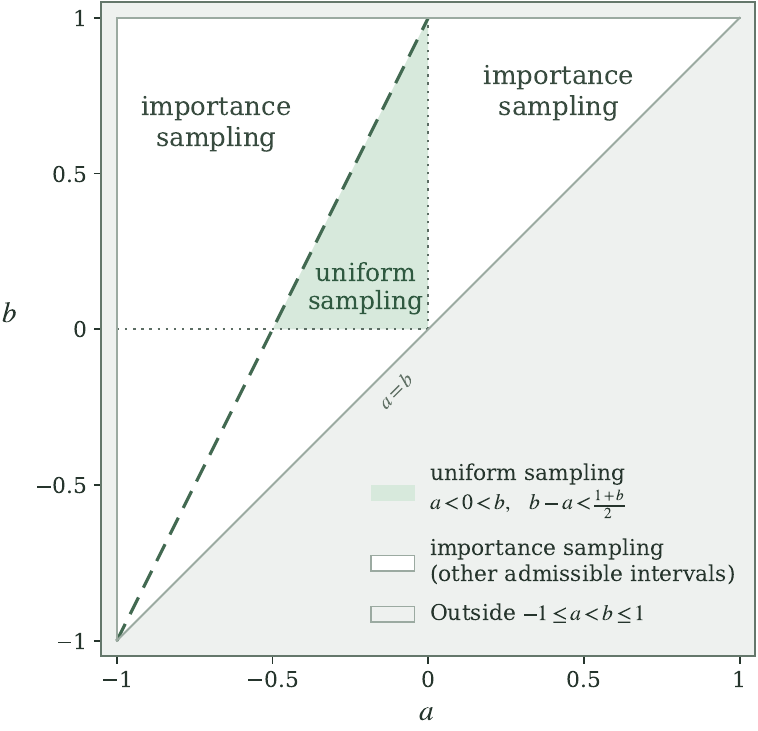}
        \captionsetup{margin={3mm,0pt}, skip=1pt}
        \caption{Uniform sampling and importance sampling regions in parameter space.}
        \label{fig:importance-sampling}
        % \vspace{-5mm}
    \end{minipage}
    \hfill
    \begin{minipage}[t]{0.59\textwidth}
        \vspace{0pt}
        \centering
        \includegraphics[width=\linewidth]{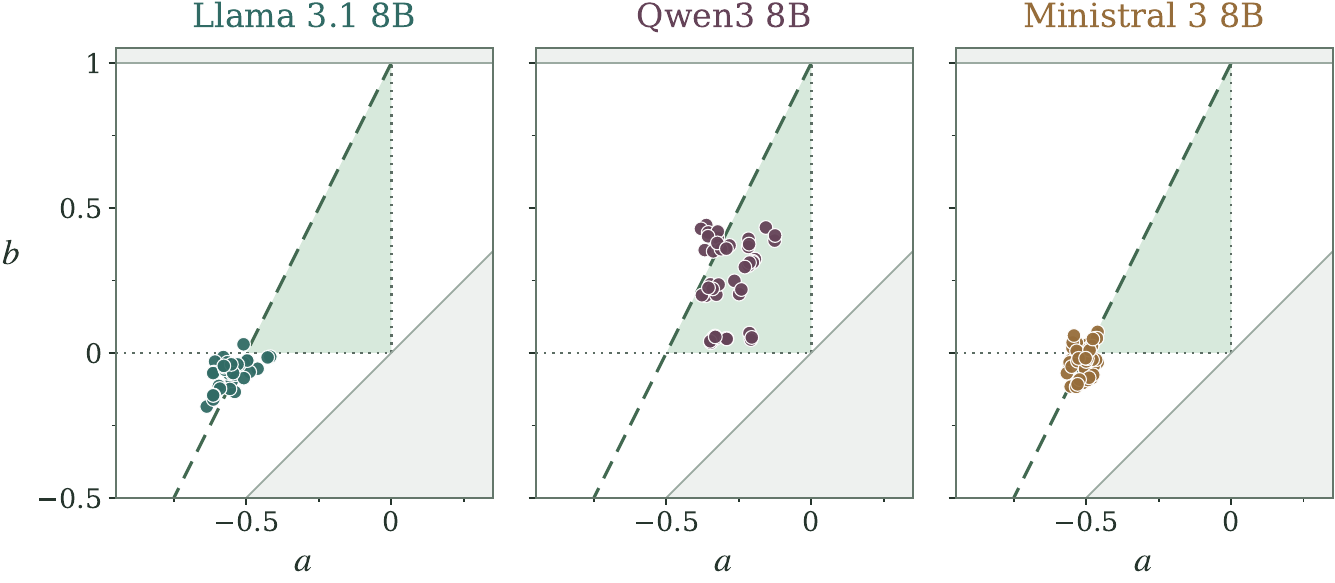}
        \captionsetup{margin={7mm,0pt}}
        \caption{Empirical bounds $a$ and $b$ for 8B models Llama 3.1, Qwen3, and Ministral 3. Each point represents one context (from LongBench v2/InfiniteBench/SCBench)--context-length (8/32/128/256 k)--layer(4/8/16/24) configuration, with endpoints $a$ and $b$ computed as $0.001$ and $0.999$ quantiles of cosine similarities distribution. Initial-token (attention sink) is excluded. See \Cref{sec:qk_experiments} for more plots.}
        \label{fig:ab_real}
        % \vspace{-5mm}
    \end{minipage}
\end{figure}

% \begin{wrapfigure}{l}{0.45\textwidth}    \centering
%     \begin{minipage}[t]{0.44\textwidth}
%             \vspace{2pt}
%             \centering
%             \includegraphics[width=\linewidth]{img/ab_optimal_exponent.pdf}
%     \caption{Optimal exponent guaranteed by \Cref{thm:attention} ($\min\{(1 + b)/2, (b - a)\}$). For details on real dataset configurations see \Cref{fig:ab_real}.}
%     \end{minipage}
% \end{wrapfigure}
% \end{comment}
 
%\paragraph{Bounded cosine similarities.} The distribution of cosine similarities of key/query pairs from the transformer architecture has received a lot of attention in the literature (see, e.g.~\cite{mazare2025saap}). It has been empirically established that, once the attention sink, namely the  first token, is removed, the rest of the correlations are small, justifying our assumption. See also \Cref{sec:qk_experiments} for our experiments on correlations and attention sinks. 

%\paragraph{Uniform sampling in theory and practice.}
Theorem~\ref{thm:attention} is achieved by uniform sampling when the key/query cosine similarities are small ($a < 0 < b$ and $b - a \leq \frac{1+b}{2}$), and by an importance sampling scheme for larger cosine similarities. Interestingly, even though uniform sampling is by no means the best token eviction policy, it is quite competitive with some earlier works on token eviction for KV cache compression, especially query-independent ones. See, for example, Table 1 in ~\cite{kochetkova2025balancekv} for a comparison of SnapKV and BalanceKV with uniform sampling. The recent work of~\cite{wang2026randomattention} provides more evidence for this phenomenon. Figure~\ref{fig:importance-sampling} helps explain this from the perspective of theory: the empirical $(a,b)$ values for widely used models lie in or close to the uniform sampling regime.

In addition to the theorem statement, the {\em proof} of the theorem provides insights into another empirical phenomenon. 
Specifically, the anisotropy of the key/query space that has been empirically observed in the literature~\cite{GodeyCS24} appears naturally in our lower bound instance. In our hard input distribution on key and query vectors, the keys and queries are similarly distributed except that the keys are positively shifted along a specific direction $e_0$ while the queries are negatively shifted along $e_0$. This construction, which is critical to establishing a matching space lower bound, is indeed anisotropic: the keys have a large positive projection on a common direction, and the queries have a large negative projection on the same direction, similarly to distributions observed in practice (\Cref{tab:qk_mean_cosine_matched}, \Cref{fig:lb-real}). See Section~\ref{sec:lower} for more details.

\begin{table}[ht]
\centering
\caption{Cosine similarity between mean query and key vectors.
Values are means $\pm$ standard deviations across example/context-length
configurations common to all models and layers.}
\label{tab:qk_mean_cosine_matched}
\renewcommand{\arraystretch}{1.1}
\small
\begin{tabular}{@{}lrrrr@{}}
\toprule
Model & Layer 4 & Layer 8 & Layer 16 & Layer 24 \\
\midrule
Llama 3.1 8B
& $-0.584 \pm 0.148$
& $-0.725 \pm 0.098$
& $-0.681 \pm 0.076$
& $-0.688 \pm 0.105$ \\
Qwen3 8B
& $0.068 \pm 0.037$
& $-0.146 \pm 0.034$
& $0.024 \pm 0.044$
& $-0.143 \pm 0.080$ \\
Ministral 3 8B
& $-0.724 \pm 0.047$
& $-0.628 \pm 0.086$
& $-0.613 \pm 0.078$
& $-0.658 \pm 0.045$ \\
\bottomrule
\end{tabular}
\vspace{-3mm}
\end{table}

% \WFclear

% \paragraph{Related work on worst case space complexity of attention approximation.}  Prior work~\citep{HarisO25,libertyak26,chen2026tight} has established nearly asymptotically right bounds on the space complexity of approximating attention in the worst case. The algorithmic landscape is quite diverse: a combination of sketching and random sampling (or, discrepancy) approaches is needed to achieve the bounds. A comparison of the bounds is given in Table~\ref{tab:related_work} below. We note that the worst case setting is crucially different from our setup in that queries can be arbitrarily well aligned with key.

\subsection{Related work}

\paragraph{Related work on worst-case space complexity of attention approximation.}
Prior work~\citep{HarisO25,libertyak26,chen2026tight} has established nearly tight asymptotic bounds on the space complexity of approximating attention in the worst case. The algorithmic landscape is diverse, with approaches drawing on sketching, random sampling, discrepancy theory, and kernel-density estimation. Provable attention approximation has also led to implemented KV cache compression methods. SubGen~\citep{zandieh2024subgen} combines online key clustering with $\ell_2$ sampling of values to obtain approximation guarantees that exploit clustering structure, and evaluates its implementation on long-context question-answering tasks. BalanceKV~\citep{kochetkova2025balancekv} develops a streaming compression algorithm based on geometric discrepancy and vector balancing, with both theoretical guarantees and empirical evaluations of attention approximation and end-to-end generation on long-context benchmarks. A comparison is given in Table~\ref{tab:related_work}. We note that the worst case setting is crucially different from our setup in that queries can be arbitrarily well aligned with key.
\paragraph{Related work on KV cache compression.} We mention the lines of the most relevant work.

\noindent{\em KV subset selection.} A prominent approach to KV-cache compression retains a subset of the
original key-value pairs. H$_2$O combines accumulated-attention heavy hitters with recent
tokens~\citep{zhang2023h2o}, while Scissorhands uses token importance
that persists across decoding steps~\citep{liu2023scissorhands}.
StreamingLLM keeps initial attention sinks and a sliding window of
recent tokens~\citep{xiao2024streamingllm}. SnapKV selects KV positions
per head before generation using a window at the prompt's
end~\citep{li2024snapkv}. PyramidKV splits a fixed cache budget across
layers~\citep{cai2025pyramidkv}, while Ada-KV splits it across heads
using an attention-output error bound~\citep{feng2025adakv}. These approaches exploit historical attention,
prompt observations, or layer/head structure, rather than establishing
optimal retained-token complexity under an explicit restriction on
future key-query similarities.

\noindent{\em Compressed KV representations.}
Another line of work reduces storage per cached vector.
KIVI quantizes keys per channel and values per
token~\citep{liu2024kivi}. KVQuant combines distribution-aware
quantization, pre-RoPE key quantization, and separate outlier
handling~\citep{hooper2024kvquant}. TurboQuant uses random rotations
for data-oblivious vector quantization, with guarantees on mean-squared
and inner-product distortion~\citep{zandieh2026turboquant}.
Palu caches low-dimensional intermediate representations from
low-rank factorizations of key/value projections~\citep{chang2025palu}.
Cartridges learns context representations through
self-study~\citep{eyuboglu2025cartridges}. These representations fall
outside our weighted-subset model, so our retained-token lower bounds
do not apply to arbitrary quantized or low-rank encodings.

\noindent{\em Random KV eviction.} \citet{wang2026randomattention} propose a method that reserves cache space for the input prompt and the most recent tokens, then fills the remaining slots by uniform random sampling, independently across KV heads at each eviction event. Their method is competitive with importance-based policies on reasoning benchmarks, providing further evidence that random sampling can be effective in practice.
 
\begin{comment}
\begin{table}[ht]
\captionsetup{skip=1pt}
\caption{Worst-case space complexity of attention. $\widetilde{O}, \widetilde{\Omega}$ suppress  $\text{poly}(\frac1{\eps}\log n)$ factors}
\label{tab:related_work}
\centering
\small
\setlength{\tabcolsep}{4pt}
\renewcommand{\arraystretch}{1.3}
\begin{tabularx}{\textwidth}{@{}l X c c@{}}
\toprule
Reference & Model / regime & Upper bound & Lower bound \\
\midrule
\cite{kochetkova2025balancekv}
& Coreset size (tokens)
& $\widetilde{O}\!\left(\sqrt{d}e^{2R^2}\right)$
& $\widetilde{\Omega}\!\left(1\right)$
\\[3pt]
\cite{HarisO25}
& Space complexity (bits) for entrywise approx
& $\widetilde O(nd)$
& ${\Omega}(nd)$
\\[3pt]
\cite{libertyak26}
& Coreset size (tokens)
& $\widetilde{O}\!\left(\sqrt{d}e^{R^2 + \log R + o(\log \log R)}\right)$
& $\widetilde{\Omega}\!\left(\sqrt{d}e^{R^2}\right)$
\\[3pt]
\cite{chen2026tight}
& Space complexity (bits)
& $\widetilde{O}\!\left({e^{(1+o(1))R^2}}\right)$
& $\widetilde{\Omega}\!\left({e^{(1-o(1))R^2}}\right)$
\\[3pt]
\textbf{This work}, $b - a < \frac{1 + b}{2}$
& Coreset size (tokens)
& $\widetilde{O}\!\left(e^{(b - a)R^2}\right)$
& $\widetilde{\Omega}\!\left(e^{(b - a - o(1))R^2}\right)$
\\[3pt]
\textbf{This work}, $b - a > \frac{1 + b}{2}$
& Coreset size (tokens)
& $\widetilde{O}\!\left(e^{\frac{1 + b}{2}R^2}\right)$
& $\widetilde{\Omega}\!\left(e^{(\frac{1 + b}{2} - o(1))R^2}\right)$
\\
\bottomrule
\end{tabularx}
\end{table}
\end{comment}

\begin{table}[ht]
\captionsetup{skip=1pt}
\caption{Worst-case space complexity of attention.
$\widetilde{O},\widetilde{\Omega}$ suppress
$\operatorname{poly}(\eps^{-1}\log n)$ factors.}
\label{tab:related_work}
\centering
\small
\setlength{\tabcolsep}{4pt}
\renewcommand{\arraystretch}{1.3}
\begin{tabularx}{\textwidth}{@{}l X c c@{}}
\toprule
Reference & Model / regime & Upper bound & Lower bound \\
\midrule
\cite{kochetkova2025balancekv}
& Coreset size (tokens)
& $\widetilde{O}\!\left(\sqrt{d}e^{2R^2}\right)$
& $\widetilde{\Omega}\!\left(1\right)$
\\[3pt]
\cite{HarisO25}\textsuperscript{a}
& Space complexity (bits)
& $\widetilde{O}(nd)$
& $\Omega(nd)$
\\[3pt]
\cite{libertyak26}\textsuperscript{b}
& Coreset size (tokens)
& $\widetilde{O}\!\left(
    \sqrt{d}e^{R^2+\log R+o(\log\log R)}
  \right)$
& $\widetilde{\Omega}\!\left(\sqrt{d}e^{R^2}\right)$
\\[3pt]
\cite{chen2026tight}\textsuperscript{c}
& Space complexity (bits)
& $\widetilde{O}\!\left(e^{(1+o(1))R^2}\right)$
& $\widetilde{\Omega}\!\left(e^{(1-o(1))R^2}\right)$
\\[3pt]
\midrule
\textbf{This work}\textsuperscript{d},
$b-a\leq \frac{1+b}{2}$
& Coreset size (tokens)
& $\widetilde{O}\!\left(e^{(b-a)R^2}\right)$
& $\widetilde{\Omega}\!\left(e^{(b-a-o(1))R^2}\right)$
\\[3pt]
\textbf{This work}\textsuperscript{d},
$b-a>\frac{1+b}{2}$
& Coreset size (tokens)
& $\widetilde{O}\!\left(e^{\frac{1+b}{2}R^2}\right)$
& $\widetilde{\Omega}\!\left(
    e^{(\frac{1+b}{2}-o(1))R^2}
  \right)$
\\
\bottomrule
\end{tabularx}

\vspace{3pt}
\begin{minipage}{\textwidth}
\footnotesize
\raggedright

\textsuperscript{a}
For entry-wise approximation guarantee.
% \par\smallskip

\textsuperscript{b}
The coreset guarantee holds simultaneously for all queries
in the prescribed norm ball.
% \par\smallskip

\textsuperscript{c}
Assuming $d=\operatorname{polylog}(n)$
and $R^2=\Omega(\log n)$.
% \par\smallskip

\textsuperscript{d}
\Cref{fig:ab_real} and \Cref{fig:cos_sim_gridplot} suggest that our bounds are $\tilde{O}(e^{\approx 0.6 R^2})$ for most practical configurations of $a$ and $b$.
\vspace{-5mm}
\end{minipage}
\end{table}

\section{Uniform sampling regime}
\subsection{Upper Bound}\label{sec:upper-uniform}

% \ying{This subsection has been rewritten in full. Though one thing we could discuss is whether we want to put the sampling procedure into an algorithm box?}\eee{I think it's fine as is, but it's not a strong preference}
The upper bound in the small cosine similarities regime is shown by a simple approach: sampling tokens uniformly.

Consider the small cosine similarities regime of
\cref{def:small_cos_sim}, where $-1 \leq a < 0 < b\le 1$ and $b - a \leq \frac{1+b}{2}$.
Our assumptions are $\|k_i\|\le r_k$, $\|q\|\le r_q$, and
\[
    a
    \le\frac{\ip{k_i}{q}}{\|k_i\|\|q\|}
    \le b.
\]
We write $R^2=r_kr_q/\sqrt d$. Set $m=\left\lceil\frac{96e^{(b-a)R^2}}{\eps^2}\right\rceil.$
If $m\ge n$, retain the full dataset. Otherwise, sample $m$
tokens independently and uniformly with replacement.
Writing $J_1,\ldots,J_m$ for their indices, return the estimate
\[
    \widehat{\Attn}(K,V,q) = \frac{\sum_{\ell=1}^m
        e^{\ip{k_{J_\ell}}{q}/\sqrt d}v_{J_\ell}}
         {\sum_{\ell=1}^m
        e^{\ip{k_{J_\ell}}{q}/\sqrt d}}.
\]
The intuition behind uniform sampling is that no individual token can account for too large a fraction of the attention mass. Indeed,  for every $i$ and every admissible query $q$, \begin{equation}\label{eq:bound-kernel} e^{\ip{k_i}{q}/\sqrt d}\in[e^{aR^2},e^{bR^2}]. \end{equation} Thus, the attention denominator is at least $ne^{aR^2}$, whereas the contribution of any individual token is at most $e^{bR^2}$. Consequently, each token accounts for at most an $e^{(b-a)R^2}/n$ fraction of the attention mass.

\begin{restatable}[Upper bound, \Cref{item:uniform}]{lemma}{upperuniform}
\label{lem:upper-uniform} The uniform sampling procedure described above retains at most 
\[O_{\eps}\left(
        \min\left\{
            n,\,
            e^{(b-a)R^2}
        \right\}
    \right)\]
tokens and gives an additive-$\eps$ approximation to attention with probability at least 2/3 for each fixed admissible
query.
\end{restatable}
We defer the full proof of \Cref{lem:upper-uniform} to \Cref{sec:up_bound_uniform_appendix}. We use Chebyshev's inequality to bound the attention denominator, and use Markov's inequality to bound the numerator of the error.

% A standard second-moment argument, using Chebyshev's inequality
% for the denominator and Markov's inequality for the squared norm
% of the numerator error, shows that this estimate approximates
% attention within additive error $\eps$ with probability at least
% $2/3$ for each fixed admissible query.
% The number of retained tokens is at most
% \[
%     \min\left\{n,
%         \left\lceil\frac{96e^{(b-a)R}}{\eps^2}\right\rceil
%     \right\},
% \]
% giving the claimed exponent $b-a$ in the small cosine
% similarities regime.
\subsection{Lower Bound}\label{sec:lower}

Recall that in our model \eqref{eq:attention-model}, the
token-eviction algorithm (randomly) selects tokens and assigns them
nonnegative weights before seeing the query. We will prove lower bounds against such algorithms.
% \eee{for the theorem name below: I've been naming the name results as "Upper/lower bound, \Cref{item:uniform}/\Cref{item:importance}". I don't have a strong opinion: if you'd prefer to write "uniform/importance sampling regime" instead, let's change throughout}\justin{on further reflection, I like Item 1/2 more}
\begin{theorem}[Lower bound, \Cref{item:uniform}]\label{thm:lb-attn}
Consider $r_k,r_q>0$, $\eps\in(0,1/2]$, and $-1\le a<b\le1$ in the uniform sampling regime (\cref{def:small_cos_sim}).
Consider $n$ and $d = \omega(\log n)$, and let $\xi = \sqrt{\frac{6\log(2n)}{d-3}} = o(1)$.
% \boris{give this regime a small alias which is not "uniform sampling regime" and use it here}\eee{why? can't we just say (see \Cref{def:small_cos_sim})?}
Pick any integer $m$ satisfying
\begin{equation}\label{eq:finite-conditions}
    2\le m\le n,\qquad m\le \mainterm{e^{(b-a)R^2}}\secondary{\cdot e^{-2\xi R^2}/32}.
\end{equation}
There is a fixed dataset $(K,V)$ of $n$ key-value pairs, with keys of norm $r_k$ and unit-norm values in $\R^d$, together with queries $q_1,\ldots,q_m$ of norm $r_q$, such that every key $k\in K$ satisfies
\[
    a\le\frac{\ip{k}{q_i}}{\|k\| \|q\|}\le b
    \qquad\text{for all }i\in[m].
\]
For this dataset, satisfying~\Cref{eq:attention-guarantee} at each
of these queries requires $\OPTattn\ge2m/3$.

\end{theorem}

Our lower bound constructs $m$ groups of identical key-value pairs, each paired with a different query.
A visualization of the construction is given in \cref{fig:lb-construction}.
For each group, we build the key and corresponding query from a set of nearly orthogonal vectors $u_1, \ldots, u_m$ (\cref{lem:code}):
\begin{restatable}[Nearly orthogonal vectors]{lemma}{code}\label{lem:code}
Consider integer $M \geq 2$ and $\xi \in (0,1)$. Let $D$ be an integer such that $D \geq 6\log(2M)/\xi^2$.
There exist $M$ unit vectors
$u_1,\ldots,u_M\in\R^D$ such that
$|\ip{u_i}{u_j}|\le\xi$ for every $i\ne j$. 
\end{restatable}
To see this, choose $u_1,\ldots,u_M$ independently and uniformly from
$\{-D^{-1/2},D^{-1/2}\}^D$, so that each vector has unit norm.
Hoeffding's inequality and a union bound show that the probability
of any pair satisfying $|\ip{u_i}{u_j}|>\xi$ is at most
$M^2e^{-D\xi^2/2}<1$ under the assumed bound on $D$.
We give the full proof in~\Cref{sec:lower_appendix}.

In order to keep all key-query cosine similarities in $[a,b]$, queries and keys are anisotropically shifted away from each other along a fixed direction $e_0$.
A pair of keys and queries in the same group $k_i, q_i$ are (a weighted version of) $u_i \pm e_0$ where the sign on $e_0$ differs for the key and query (along with a normalization term to set their respective norms to $r_k, r_q$).

Each query $q_i$ has similarity $b$ with its own group's key $k_i$, while the near orthogonality caps its similarity close to $a$ with the other groups' keys $k_j$. The exponential in attention amplifies this gap, giving the query's own group almost all the attention mass. We also choose nearly orthogonal value vectors, so a convex combination of other groups' value vectors cannot well-approximate attention. The
algorithm must therefore retain a key-value pair from each group with
constant probability, otherwise it would fail with probability more than $1/3$ on some group's query. 

The resulting construction holds an intriguing connection to patterns in real data. In \cref{fig:lb-real}, we see that queries and keys are separated along a common direction and otherwise hold a somewhat similar distribution though with different scales for queries and keys. These qualities also appear in our hard instance (\cref{fig:lb-construction}).

\subsubsection{Lower bound construction}

\begin{figure}[t]
    \centering
    \begin{minipage}[t]{0.47\textwidth}
        \centering
        \begin{minipage}[c][5cm][b]{\linewidth}
            \centering
            \includegraphics[
                width=\linewidth,
                height=5cm,
                keepaspectratio
            ]{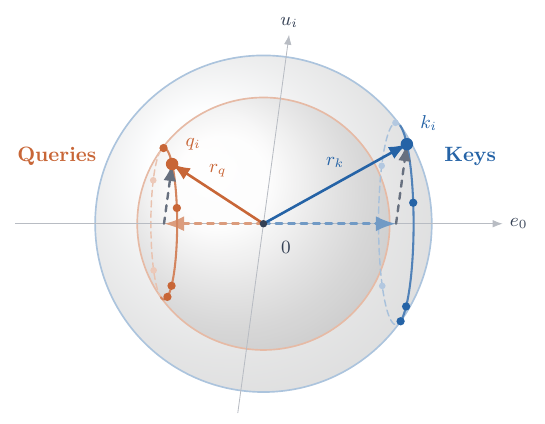}
        \end{minipage}
        \caption{
        A visualization of the hard instance in our lower bound construction. The keys and queries are shifted away from each other along direction $e_0$. Orthogonal to $e_0$, each key-query pair $(k_i, q_i)$ follows its own common direction $u_i$.
        }
        \label{fig:lb-construction}
    \end{minipage}
    \hfill
    \begin{minipage}[t]{0.47\textwidth}
        \centering
        \begin{minipage}[c][5cm][b]{\linewidth}
            \centering
            \includegraphics[
                width=\linewidth,
                height=4.5cm,
                keepaspectratio
            ]{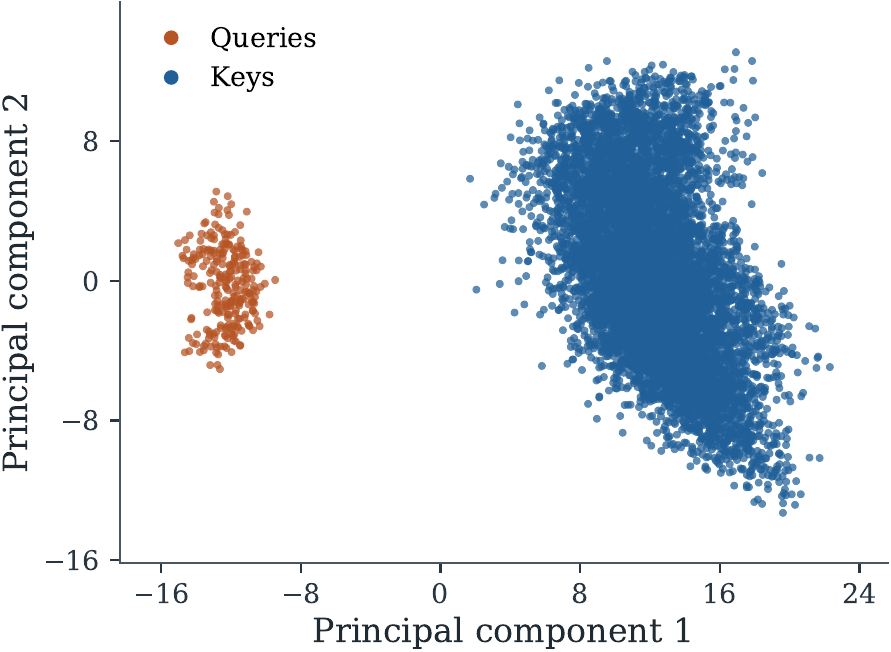}
        \end{minipage}
        \caption{Query-key geometry in Llama 3.1 8B on LongBench v2 example (8K context, layer 8, head 0). A joint PCA is fitted with equal total weight for queries and keys. Initial-token (attention sink) is excluded. See \Cref{sec:qk_experiments} for further plots.} 
        \label{fig:lb-real}
    \end{minipage}
\end{figure}

% \michael{I am a big fan of upper bounding $\varepsilon$ by $1/2$ and not writing $1/(1-\varepsilon)$ anywhere, $\varepsilon$ close to $1$ is not very interesting}

% \justin{todo: connect to main thm by setting $\rho$.}

% \subsubsection{Key, Query, Value Vectors}

% \boris{why $\sqrt{-\mu}?..$ This part needs a main-part friendly exposition}

\paragraph{Goal.}
We construct keys and queries whose scaled inner products are $bR^2$
for matching pairs $(k_i,q_i)$ and $a(1-o(1))R^2$ for nonmatching
pairs $(k_j,q_i)$, where $R^2=r_kr_q/\sqrt d$.
In~\Cref{sub:proof_lb}, we use this gap to make each query concentrate
its attention on copies of its matching key, forcing an accurate
summary to retain many key groups.

\paragraph{Construction.}
Let $e_0,e_K,e_Q$ be the first three standard basis vectors.
Choose $m$ unit vectors satisfying \cref{lem:code} in dimension
$D=d-3$, with pairwise inner products bounded in magnitude by $\xi = \sqrt{\frac{6\log(2n)}{D}}$.
Embed them into the last $d-3$ coordinates of $\R^d$, and denote the
resulting vectors by $u_1,\ldots,u_m$. Define 
\begin{equation}\label{eq:lb-geometry}%\label{eq:geometry}
\begin{aligned}
    k_i&=r_k\bigl(\underbrace{\sqrt{\alpha}\,u_i}_{\text{Shared group direction}}+\underbrace{\sqrt{\beta}\,e_0}_{\text{Anisotropic shift}} + \quad \underbrace{\sqrt{1-\alpha-\beta}\,e_K}_{\text{Norm padding}}\; \bigr), \\
    q_i&=r_q\bigl(\underbrace{\sqrt{\alpha}\,u_i}_{\text{Shared group direction}}-\underbrace{\sqrt{\beta}\,e_0}_{\text{Anisotropic shift}} + \quad \underbrace{\sqrt{1-\alpha-\beta}\,e_Q}_{\text{Norm padding}}\; \bigr),
    % q_i&=r_q\bigl(\sqrt{\alpha}\,u_i
      % -\sqrt{|\beta|}\,e_0\bigr)+ r_q\sqrt{1-\alpha-|\beta|}\,e_Q,
\end{aligned}
\end{equation}
where $\alpha\approx b-a$ and $\beta\approx -a$. Specifically, we set
\begin{equation}\label{eq:geometry-parameters-main}
    \alpha=\frac{b-a}{1+\xi},
    \qquad
    \beta=\frac{-a-\xi b}{1+\xi}.
    % h=bR,\qquad
    % L=\min\left\{\frac{R+h}{2},\frac{(b-a)R}{1+\xi}\right\},
    % \qquad \beta=L-h.
\end{equation}
We have $\alpha\ge0$ since $b\ge a$. Since $a<0$ and $b$ are fixed
and $\xi=o(1)$, for sufficiently large $n$ we have $\xi b<-a$,
so $\beta>0$. Moreover, $\alpha+\beta=2\alpha-b\le b-2a\le1.$
Therefore, $1 - \alpha - \beta \geq 0$ and the keys and queries are well defined with norms exactly equal to $\|k_i\| = r_k$ and $\|q_i\| =r_q$.
The components along $u_i$ and $e_0$ serve the purpose of achieving the desired inner products while the padding along $e_K$ and $e_Q$ increases the norms of the keys/queries to fix the denominator of cosine similarity.

\paragraph{Verifying the inner products.}
For every $i,j\in[m]$, $\frac{\ip{k_j}{q_i}}{\sqrt d}
    =\bigl(\alpha\ip{u_j}{u_i}-\beta\bigr)R^2.$
For matching pairs, $\alpha-\beta=b$ gives  
\begin{equation}\label{eq:matching_cosine}
    \frac{\ip{k_i}{q_i}}{\sqrt d}=bR^2.
\end{equation}
For nonmatching pairs, $|\ip{u_j}{u_i}|\le\xi$ gives
\begin{equation}\label{eq:non-matching_cosine}
    aR^2
    \le \frac{\ip{k_j}{q_i}}{\sqrt d}
    \le (a+2\alpha\xi)R^2
    =aR^2 + \frac{2\xi}{1+\xi}(b-a)R^2.
\end{equation}

\paragraph{Constructing the value vectors.}
Apply \cref{lem:code} a second time to obtain $m$ unit value vectors
$v_1,\ldots,v_m\in\R^d$ satisfying
\begin{equation}\label{eq:lb-value-inner-products}
    |\ip{v_i}{v_j}|\le\frac18
    \qquad\text{for all }i,j\in[m]\text{ with }i\ne j.
\end{equation}
Since $m\le n$, the assumption $d=\omega(\log n)$ ensures that this
is possible.

The final proof of the coreset size lower bound using these key, query and value vectors is given in \cref{sec:lower_appendix}.

\section{Importance sampling regime}
\subsection{Upper Bound}\label{sec:upper_importance}
When the interval $[a,b]$ is wide, the exponential in attention
can make a few tokens much more influential than the rest. To account for this, 
uniform sampling would need to retain many tokens. We thus use {\it importance sampling} instead.
Unlike uniform sampling, which treats all tokens equally,
importance sampling gives higher sampling probabilities to
tokens that can contribute more to attention. This is a
standard idea in randomized linear algebra through leverage-score
sampling~\cite{DMMW12} and in coreset construction through
sensitivity sampling~\cite{BFL21}.

Since we sample before seeing the query, we want to sample a token with probability proportional to the  largest possible attention it could receive.
This quantity is its \emph{sensitivity}: for a nonempty query
set $Q$ containing all admissible queries, define
\begin{equation}\label{eq:sens}
    \sigma_i
    =\sup_{q\in Q}
       \frac{e^{\ip{k_i}{q}/\sqrt d}}
            {\sum_{j=1}^n e^{\ip{k_j}{q}/\sqrt d}}.
\end{equation}
% Their sum $\sum^{n}_{i=1} \sigma_i$ is called the {\it total sensitivity}.

Unfortunately, these sensitivities may be difficult to compute exactly.
Instead, it suffices to use sampling scores $s_i\ge \sigma_i$, i.e., upper bounds of each sensitivities. In the sampling procedure, we sample
each token $i$ with probability $\pi_i=s_i/S$, where
$S=\sum_{i=1}^n s_i$. We set the number of samples proportional to $S$,
so smaller sampling scores  yield a smaller sample size. This leaves two
tasks: compute valid sampling scores  $s_i\ge \sigma_i$ and show that their sum $S$ is
small.

We formalize the sampling procedure in
\Cref{alg:upper-weighted-subset}. For each sampled token \(i\), we retain the corresponding key-value pair \((k_i,v_i)\) together with a positive weight. If we sample the same token
more than once, we keep one copy and add the weights together.

\begin{algorithm}[H]\caption{Attention compression by score-based sampling and reweighting}
\label{alg:upper-weighted-subset}
\begin{algorithmic}[1]
    \Require key-value pairs $(k_1,v_1),\ldots,(k_n,v_n)$,
    accuracy $0<\eps<1$, and sampling scores $s_1,\ldots,s_n>0$
    \State $S\gets\sum_{i=1}^n s_i$, \quad
    $m\gets\lceil 96S/\eps^2\rceil$
    \If{$m\ge n$}
        \State \Return $\{(k_i,v_i,1):i\in[n]\}$
    \EndIf
    \State $\pi_i\gets s_i/S$ for every $i\in[n]$
    \State Subsample $m$ tokens using probabilities $\pi_i$
        \Comment{with replacement}
    \State $c_i\gets$ number of times token $i$ was sampled
    \State \Return $\{(k_i,v_i,c_i/(m\pi_i)):c_i>0\}$
        \Comment{one copy per token}
\end{algorithmic}
\end{algorithm}

We refer to the set returned by \cref{alg:upper-weighted-subset} as the {\it summary}. Given a query $q$, the summary produces an estimator
\begin{equation}\label{eq:upper-attention-estimator}
    \widehat{\Attn}(K,V,q)
    =\frac{\sum_{i\in I}w_i e^{\ip{k_i}{q}/\sqrt d}v_i}
           {\sum_{i\in I}w_i e^{\ip{k_i}{q}/\sqrt d}},
\end{equation}
where $I := \{i: c_i > 0\}$ denotes the set of retained token indices.
The denominator is positive, since the summary is nonempty
and all retained weights are positive.

We first state the approximation guarantee of \cref{alg:upper-weighted-subset}. Using a number of tokens proportional to $S$, the estimator from \cref{{eq:upper-attention-estimator}} approximates the attention with high probability.

\begin{restatable}[Approximation guarantee]{lemma}{samplingguarantee}
\label{lem:upper-sampling}
Suppose $\|v_i\|\le1$ and $s_i\ge\sigma_i$ for every $i\in[n]$.
Then \cref{alg:upper-weighted-subset} returns a weighted attention
summary of size at most
\[
    \min\left\{n,
        \left\lceil\frac{96S}{\eps^2}\right\rceil
    \right\}.
\]
For every fixed query $q\in Q$, its estimate satisfies
\[
    \Pr\!\left[
        \bigl\|\widehat{\Attn}(K,V,q)-\Attn(K,V,q)\bigr\|
        \le\eps
    \right]\ge\frac23,
\]
where the probability is over the sampling in
\cref{alg:upper-weighted-subset}.
\end{restatable}

We defer the full proof of \Cref{lem:upper-sampling} to \Cref{sec:up_bound_importance_appendix}. Similar to the proof of \Cref{lem:upper-uniform}, we use Chebyshev's inequality to bound the attention denominator, and use Markov's inequality to bound the numerator of the error.

\Cref{lem:upper-sampling} reduces the problem to constructing sampling scores $s_i$ whose sum $S$ is small. To this end, we first bound the total sensitivity $\sum_{i=1}^n\sigma_i$ of any query set for which all key-query inner products lie in a fixed interval. This condition is satisfied by our query set $Q$. However, $Q$ need not be convex, making it difficult to compute sampling scores. We therefore partition the keys and possible query norms into small ranges
and, for each pair of ranges, replace the cosine-similarity constraints
by fixed inner-product bounds. This partitions the problem into a collection of smaller convex relaxations on which we can approximate the sensitivities. Applying the
total-sensitivity bound to these relaxations allows us to control the
sum $S$. These two steps are formalized in
\cref{lem:total-sensitivity,lem:computing-sampling-scores}.

\begin{restatable}[Total sensitivity bound]{lemma}{totalsens}
\label{lem:total-sensitivity}
Suppose $r_k,r_q>0$ and $\|k_i\|\le r_k$ for every $i$,
and let $\mathcal Q$ be a nonempty compact query set such that
\[\mathcal{Q} \coloneqq \Bigl\{q \in \R^d:  \|q\|\le r_q, \ 
\ell\le\ip{k_i}{q}/\sqrt{d} \leq h \quad\forall i\in[n]\Bigr\}.
\]
For sensitivities $\sigma_i$ defined in \Cref{eq:sens} with the supremum taken over $q \in \mathcal{Q}$, 
we have
\begin{equation}\label{eq:sens-bound}
    \sum_{i=1}^n\sigma_i\le\exp\!\left(
        \min\left\{h-\ell,\frac{R^2+h}{2}\right\}
    \right).
\end{equation}
\end{restatable}
At a high level, 
the two terms in \cref{eq:sens-bound}
come from two different ways of controlling the total sensitivity.
The first term $h-\ell$ follows from the fact that $\sigma_i\le e^h/(ne^\ell)$.
For the second term, we pair each key $k_i$ with a query $q_i$ attaining
its sensitivity and assign weights to these pairs proportionally to their
sensitivities. Writing $\bar k,\bar q$ for the
weighted mean key and query, using Jensen's inequality, the proof gives
\[
    \log\sum_i\sigma_i
    \le h-\frac{\ip{\bar k}{\bar q}}{\sqrt d},
    \qquad
    \log\sum_i\sigma_i
    \le R^2+\frac{\ip{\bar k}{\bar q}}{\sqrt d}.
\]
Averaging these two bounds yields
$\log\sum_i\sigma_i\le(R^2+h)/2$.
The full proof is deferred to \Cref{sec:up_bound_importance_appendix}.

It remains to construct the sampling scores.

\begin{restatable}[Computing sampling scores]{lemma}{computescores}
\label{lem:computing-sampling-scores}
Let $-1\le a<b\le1$ and $r_k,r_q>0$, and write $R^2=\frac{r_kr_q}{\sqrt d},
    \kappa=\min\left\{b-a,\frac{1+b}{2}\right\}.$
Suppose $\|k_i\|\le r_k$ for every $i\in[n]$, and let
\[
    \mathcal Q\coloneqq
    \left\{q\in\R^d:
        \|q\|\le r_q,\quad
        a\|k_i\|\|q\|
        \le\ip{k_i}{q}
        \le b\|k_i\|\|q\|
        \quad\forall i\in[n]
    \right\}.
\]
For sensitivities $\sigma_i$ defined in \Cref{eq:sens} with the supremum taken over $q \in \mathcal{Q}$,
we can compute sampling scores $s_i$ satisfying $s_i \geq \sigma_i$ for every $i\in[n]$, and
    \[\sum_{i=1}^n s_i
    \le\frac{128e^2}{\kappa^2}e^{\kappa R^2}.\]
Moreover, such sampling scores can be obtained by solving $n\cdot \max\{1,\lceil R^2\rceil\}$ convex optimization problems, using only the keys and the parameters \(r_k,r_q,a,b\).
\end{restatable}

%The intuition behind the proof is that our original query set
%$Q$ (\cref{eq:set}) need not be convex because of the cosine-similarity
%constraint. To fix this, we group keys and queries into narrow norm ranges and
%use their endpoints to lower-bound $a\|k_i\|\|q\|$ and upper-bound
%$b\|k_i\|\|q\|$. This replaces the cosine constraints by linear
%inequalities in $q$, which, together with the query-norm upper bound,
%define a convex relaxation for each pair of groups. We compute local
%sensitivity bounds by convex optimization and combine them into
%sampling scores, applying \Cref{lem:total-sensitivity} to each
%relaxation to control their total sum.
The full proof is deferred to \Cref{sec:up_bound_importance_appendix}. 

% \begin{proof}[Proof sketch]
% The main difficulty is that the admissible query region
% $\mathcal Q$ need not be convex. We cover it with explicit
% convex regions, allowing extra queries so that we can compute
% sensitivity upper bounds by convex optimization. Allowing
% extra queries keeps the scores valid, but could make them
% too large. We control this loss by choosing separate regions
% for different key and query norms and applying \cref{lem:sens}
% to bound the sum of the resulting scores.
% \end{proof}

% \eee{do we want the remark below?}
% \begin{remark}[Numerical evaluation and preprocessing]
% Only constant additive accuracy is needed in these optimization
% problems. For numerical evaluation, one can reserve part of the
% $\log2$ error budget for evaluating the objective: for each
% nonempty $\widetilde{\mathcal Q}_{uv}$ and $i\in I_u$, it
% suffices to obtain a number $\xi_i^{uv}$ satisfying
% \[
%     \min_{z\in\widetilde{\mathcal Q}_{uv}}f_i^u(z)
%     \le\xi_i^{uv}
%     \le\min_{z\in\widetilde{\mathcal Q}_{uv}}f_i^u(z)+\log2
% \]
% and set $s_i^{uv}=2e^{-\xi_i^{uv}}$.
% For empty domains, set $s_i^{uv}=0$, and combine the local
% scores as $s_i=\sum_{v=1}^M s_i^{uv}$.

% The construction in \cref{lem:computing-sampling-scores}
% uses at most $nM$ optimization problems, where
% $M=\max\{1,\lceil R^2\rceil\}$, together with feasibility
% checks and any necessary preprocessing for lower-dimensional
% domains.
% \end{remark}

We now combine the approximation guarantee (\cref{lem:upper-sampling}) with the sampling score bound (\cref{lem:total-sensitivity,lem:computing-sampling-scores})
to prove the upper bound in the importance sampling regime.

\begin{corollary}[Upper bound, \Cref{item:importance}]
\label{cor:upper-importance}
Under the assumptions of \cref{lem:computing-sampling-scores},
fix $a,b$ independently of $n$, suppose $\|v_i\|\le1$ for every
$i\in[n]$, and let $0<\eps<1$.
Then \cref{alg:upper-weighted-subset}, with the scores constructed
there, retains at most
\[
    O_{\eps}\left(
        \min\left\{
            n,\,
            e^{(b-a)R^2},\,
            e^{\frac{1+b}{2}R^2}
        \right\}
    \right)
\]
tokens and gives an additive-$\eps$ approximation to attention
with probability at least $2/3$ for each fixed admissible query.
\end{corollary}

\begin{proof}
By \cref{lem:computing-sampling-scores}, the sampling scores
satisfy $s_i\ge\sigma_i$ and
$S\le128e^2\kappa^{-2}e^{\kappa R^2}$, where
$\kappa=\min\{b-a,(1+b)/2\}$.
Since $a,b$ are fixed, $\kappa^{-2}$ is a constant.
Substituting into \cref{lem:upper-sampling} and absorbing $\kappa^{-2}$ into the implicit constant proves the upper bound.
\end{proof}

\subsection{Lower bound}\label{sec:importance_lower}

The lower bound in the importance sampling regime uses the same
construction and analysis as in the uniform sampling regime
(\cref{sec:lower}). The only change in the theorem statement is that $m$ must satisfy
\[
    2\le m\le n,\qquad
    m\le \min\left\{
        \mainterm{e^{\frac{1+b}{2}R^2}},
        \mainterm{e^{(b-a)R^2}}
    \right\}\secondary{\cdot e^{-2\xi R^2}/32}.
\]
The full statement (\Cref{thm:lb-attn-full}) and proof are given
in~\cref{sec:importance_lower_appendix}.

% \michael{I am a big fan of upper bounding $\varepsilon$ by $1/2$ and not writing $1/(1-\varepsilon)$ anywhere, $\varepsilon$ close to $1$ is not very interesting}

\section{Open Problems}
The following questions naturally arise from our work:
\begin{itemize}[topsep=0.5em, itemsep=0.6em, parsep=0pt]
    \item Can we provably explain why transformers learn this geometry?
    In particular, under what assumptions on the data and training
    dynamics can we prove that small key--query cosine similarities
    and negative alignment between mean keys and queries emerge?

    \item Do real KV caches have other structure that our model misses?
    Our hard instances share some geometric features with trained models,
    but real caches may have additional properties that make them easier
    to compress.
\end{itemize}

\section*{Acknowledgments}
JC,YF, PI, were supported in part by the NSF TRIPODS program (award DMS-2022448) and Wistron Corporation.
Ekaterina Kochetkova and Boris Prokhorov were supported by the Swiss National Science Foundation grant 10004935 - Sublinear Algorithms for Non-linear Structures (Project MINT)

\subsection*{AI use statement}

% (This section is \textbf{required} and does not count toward the page limit.)

In this work, we used generative AI tools to provide critical ingredients for proving mathematical claims in the importance sampling regime, assist in the writing of proofs, propose or refine hypotheses and implement methods.

We have not used generative AI tools to help
formulate mathematical claims, develop theoretical models or conceptual frameworks, design or provide feedback on research methodology or experiments, interpret results, support qualitative and thematic data analysis. Generating synthetic data sets, cleaning and reformatting datasets and
assisting with translation are not applicable to this work.

Additionally, we used generative AI tools for creating scientific figures, suggesting experimental parameters, creating software code, drafting parts of the paper, sourcing/searching for information, editing to improve readability, proposing a title for the paper.

We have reviewed all AI-assisted work. We
checked LLM-generated parts of the proof and rewrote it ourselves. We take responsibility for the final content of this work, including text, claims or artifacts produced with the aid of generative AI.

\bibliography{main} 
\bibliographystyle{iclr2027_conference}

\newpage
\appendix
\section{Preliminaries}
\subsection{Standard Inequalities and Identities}\label{sec:preliminaries_inequalities}

We record several standard inequalities and identities used in the proofs.

\begin{lemma}[Markov's inequality]
\label{lem:markov}
Let $X$ be a nonnegative random variable. Then, for every $t>0$,
\[
    \Pr[X\ge t]\le \frac{\E[X]}{t}.
\]
In particular, for any random vector $X$ with finite second moment,
\[
    \Pr[\|X\|\ge t]
    \le
    \frac{\E\|X\|^2}{t^2}.
\]
\end{lemma}

\begin{lemma}[Chebyshev's inequality]
\label{lem:chebyshev}
Let $X$ be a real-valued random variable with finite variance. Then,
for every $t>0$,
\[
    \Pr\!\left[|X-\E X|\ge t\right]
    \le
    \frac{\Var(X)}{t^2}.
\]
\end{lemma}

\begin{lemma}[Hoeffding's inequality]\label{lem:hoeffding}
Let $X_1,\ldots,X_n$ be independent random variables such that $X_i \in [a,b]$ for every $i$. Define \[ \bar X := \frac{1}{n}\sum_{i=1}^n X_i. \] Then, for every $\varepsilon > 0$, \[ \Pr\!\left( \left|\bar X - \mathbb{E}[\bar X]\right| \geq \varepsilon \right) \leq 2\exp\!\left( -\frac{2n\varepsilon^2}{(b-a)^2} \right). \] \end{lemma}

% \begin{lemma}[Union bound]
% \label{lem:union-bound}
% For any events $E_1,\ldots,E_m$,
% \[
%     \Pr\!\left[\bigcup_{\ell=1}^m E_\ell\right]
%     \le
%     \sum_{\ell=1}^m \Pr[E_\ell].
% \]
% \end{lemma}

\begin{lemma}[Jensen's inequality]
\label{lem:jensen-log}
Let $\lambda_1,\ldots,\lambda_n>0$ satisfy
$\sum_{i=1}^n\lambda_i=1$, and let $x_1,\ldots,x_n>0$. Then
\[
    \log\!\left(\sum_{i=1}^n x_i\right)
    =
    \log\!\left(
        \sum_{i=1}^n
        \lambda_i\frac{x_i}{\lambda_i}
    \right)
    \ge
    \sum_{i=1}^n
    \lambda_i\log\!\left(\frac{x_i}{\lambda_i}\right).
\]
\end{lemma}

\begin{lemma}[Cauchy-Schwarz inequality]
\label{lem:weighted-cauchy-schwarz}
Let $\lambda_1,\ldots,\lambda_n\ge0$, and let
$x_1,\ldots,x_n$ and $y_1,\ldots,y_n$ be vectors in a Euclidean space.
Then
\[
    \left|
        \sum_{i=1}^n
        \lambda_i\langle x_i,y_i\rangle
    \right|
    \le
    \left(
        \sum_{i=1}^n\lambda_i\|x_i\|^2
    \right)^{1/2}
    \left(
        \sum_{i=1}^n\lambda_i\|y_i\|^2
    \right)^{1/2}.
\]
\end{lemma}

\begin{lemma}[Arithmetic-geometric mean inequality]
\label{lem:am-gm}
For every $x,y\ge0$,
\[
    \sqrt{xy}\le \frac{x+y}{2}.
\]
\end{lemma}

\begin{comment}

\begin{lemma}[Inner-product lower bound]
\label{lem:inner-product-bound}
For any vectors $x,y$ in a Euclidean space,
\[
    \langle x,y\rangle
    \ge
    -\frac{\|x\|^2+\|y\|^2}{2}.
\]
\end{lemma}

\begin{lemma}[Second moment identities]
\label{lem:second-moment}
Let $X_1,\ldots,X_m$ be independent random vectors with finite second
moments and $\E X_\ell=0$ for every $\ell$. Then
\[
    \E\left\|
        \frac1m\sum_{\ell=1}^m X_\ell
    \right\|^2
    =
    \frac1{m^2}
    \sum_{\ell=1}^m \E\|X_\ell\|^2.
\]
In particular, if the $X_\ell$ are identically distributed, then
\[
    \E\left\|
        \frac1m\sum_{\ell=1}^m X_\ell
    \right\|^2
    =
    \frac1m \E\|X_1\|^2.
\]
For scalar random variables, the corresponding identity is
\[
    \Var\!\left(
        \frac1m\sum_{\ell=1}^m X_\ell
    \right)
    =
    \frac1{m^2}
    \sum_{\ell=1}^m \Var(X_\ell),
\]
and hence equals $\Var(X_1)/m$ in the identically distributed case.
\end{lemma}
\end{comment}

\begin{lemma}[Weighted centering]
\label{lem:weighted-variance}
Let $w_1,\ldots,w_n\ge0$ and let $x_1,\ldots,x_n$ be vectors in a Euclidean space. Let
\[
    W=\sum_{i=1}^n w_i>0,
    \qquad
    \bar x=\frac1W\sum_{i=1}^n w_i x_i.
\]
Then
\[
    \sum_{i=1}^n w_i\|x_i-\bar x\|^2
    =
    \sum_{i=1}^n w_i\|x_i\|^2
    -
    W\|\bar x\|^2.
\]
In particular, if $\sum_i w_i=1$, then
\[
    \sum_{i=1}^n w_i\|x_i-\bar x\|^2
    =
    \sum_{i=1}^n w_i\|x_i\|^2-\|\bar x\|^2.
\]
\end{lemma}

% \begin{lemma}[Perturbation of a ratio]
% \label{lem:ratio-perturbation}
% Let $y,\widehat y$ be vectors and let $z,\widehat z>0$. Then
% \[
%     \left\|
%         \frac{\widehat y}{\widehat z}
%         -
%         \frac{y}{z}
%     \right\|
%     \le
%     \frac{\|\widehat y-y\|}{\widehat z}
%     +
%     \frac{\|y\|\,|\widehat z-z|}{z\widehat z}.
% \]
% \end{lemma}

\subsection{Convex Optimization}\label{sec:preliminaries_convex}

We record the convex optimization facts used to compute the sampling
scores. The optimization guarantee below is stated in the oracle model.

\begin{lemma}
\label{lem:log-sum-exp}
Let $u_1,\ldots,u_N\in\mathbb R^d$ and
$c_1,\ldots,c_N\in\mathbb R$. The function
\[
    f(x)
    =\log\!\left(
        \sum_{j=1}^N e^{\ip{u_j}{x}+c_j}
    \right)
\]
is convex and continuously differentiable, with
\[
    \nabla f(x)
    =
    \frac{\sum_{j=1}^N e^{\ip{u_j}{x}+c_j}u_j}
         {\sum_{j=1}^N e^{\ip{u_j}{x}+c_j}}.
\]
\end{lemma}

The convexity statement follows from convexity of log-sum-exp and
composition with an affine map; see
\citet[Sections~3.1.5 and~3.2.2]{VB04}.
The gradient formula follows by differentiation.

\begin{lemma}
\label{lem:norm-affine-constraints}
A set defined by finitely many affine inequalities and constraints
of the form
\[
    \|Ax+c\|\le \ip{u}{x}+\beta
\]
is closed and convex. If it is bounded, it is compact.
\end{lemma}

\begin{proof}
The function
$x\mapsto \|Ax+c\|-\ip{u}{x}-\beta$
is continuous and convex. Thus its non-positive sublevel set is closed
and convex, and these properties are preserved by intersection.
Compactness follows from closedness and boundedness.
\end{proof}

\begin{lemma}[Additive-accuracy convex minimization]
\label{lem:additive-convex-minimization}
Let $C\subseteq\mathbb R^d$ be a nonempty compact convex set, and let
$f:C\to\mathbb R$ be continuous and convex. Then, for every $\eta>0$,
the convex optimization problem
\[
    \min_{x\in C} f(x)
\]
admits an additive-$\eta$ solution; that is, there exists
$\widehat x\in C$ satisfying
\[
    \min_{x\in C}f(x)
    \le f(\widehat x)
    \le \min_{x\in C}f(x)+\eta.
\]
\end{lemma}

\begin{proof}
Since $C$ is compact and $f$ is continuous, $f$ attains its minimum
on $C$. An exact minimizer is, in particular, an additive-$\eta$
solution.
\end{proof}

\section{Proofs from \Cref{sec:upper-uniform}}\label{sec:up_bound_uniform_appendix}

\upperuniform*

\begin{proof}
If the algorithm retains all tokens, the estimate is exact. Otherwise, fix an admissible query $q$, we write the true average and the empirical average of kernels as
\[
D = \frac1n\sum_i \exp(\ip{k_i}{q}/\sqrt d)
\qquad \text{and} \qquad
\widehat D =\frac1m\sum_{\ell=1}^m \exp(\ip{k_{J_\ell}}{q}/\sqrt d).
\]
We will show that 
\[
    \Pr\left[
        \left\|
        \widehat{\Attn}(K,V,q)-\Attn(K,V,q)
        \right\|\le\eps
    \right]\ge\frac23, \quad \text{where} \quad 
    \widehat{\Attn}(K,V,q) := \frac{\sum_{\ell=1}^m
        e^{\ip{k_{J_\ell}}{q}/\sqrt d}v_{J_\ell}}
         {\sum_{\ell=1}^m
        e^{\ip{k_{J_\ell}}{q}/\sqrt d}}
\]
for uniform samples $J_1,\cdots, J_m$. 
We write the attention error vector as 
\[
\widehat{\Attn}(K,V,q)-\Attn(K,V,q)
=
\frac{\frac1m\sum_{\ell=1}^m
\exp(\ip{k_{J_\ell}}{q}/\sqrt d)
\bigl(v_{J_\ell}-\Attn(K,V,q)\bigr)}{\widehat D}.
\]
We will use Chebyshev's inequality to bound the denominator, and use Markov's inequality to bound the norm of the numerator.
By \cref{eq:bound-kernel}, we have $D\ge e^{aR^2}$.
Moreover, since $J_1, \ldots, J_m$ are independent,
\[
    \Var(\widehat D)
    = \frac{1}{m}
      \Var\!\left(
          e^{\ip{k_J}{q}/\sqrt d}
      \right)
    \le
    \frac{1}{m}
    \E\!\left[e^{2\ip{k_J}{q}/\sqrt d}\right]
    \le
    \frac{e^{bR^2}D}{m},
\]
where $J$ is uniform on $[n]$. 
Thus, by Chebyshev's inequality (\cref{lem:chebyshev}),
\[
    \Pr[\widehat D<D/2]\le 4e^{bR^2}/(mD)\le 4e^{(b-a)R^2}/m
\le \eps^2/24\le 1/96.
\]

On the other hand, we can bound
the expected squared norm of the numerator as
\begin{align*}
\E \Bigl\|\frac1m\sum_{\ell=1}^m&
\exp(\ip{k_{J_\ell}}{q}/\sqrt d)
\bigl(v_{J_\ell}-\Attn(K,V,q)\bigr)\Bigr\|^2 \\
&=
\frac1{mn}\sum_i
\exp(2\ip{k_i}{q}/\sqrt d)
\bigl\|v_i-\Attn(K,V,q)\bigr\|^2 \\
&\overset{\text{\cref{eq:bound-kernel}}}{\le}
\frac{e^{bR^2}}{m}\cdot\frac1n\sum_i
\exp(\ip{k_i}{q}/\sqrt d)
\bigl\|v_i-\Attn(K,V,q)\bigr\|^2 \\
&\overset{\text{\cref{lem:weighted-variance}}}{=}
\frac{e^{bR^2}}{m}
\left(
\frac1n\sum_i
\exp(\ip{k_i}{q}/\sqrt d)\|v_i\|^2
-
D\|\Attn(K,V,q)\|^2
\right) \\
&\le \frac{e^{bR^2}D}{m}.
\end{align*}
Hence, by Markov's inequality (\cref{lem:markov}), the probability that the
numerator has norm larger than $\eps D/2$ is at most
$4e^{bR^2}/(m\eps^2D)\le 1/24$. 

Therefore,  with probability at least
$1-1/96-1/24>2/3$, we have simultaneously $\widehat D\ge D/2$ and numerator
norm at most $\eps D/2$, which implies
$\|\widehat{\Attn}(K,V,q)-\Attn(K,V,q)\|\le\eps$.
\end{proof}

\section{Proofs from \Cref{sec:lower}}\label{sec:lower_appendix}

Here we prove a standard lemma which is used in the lower bound about packing nearly orthogonal vectors in high dimensions.

\code*

\begin{proof}
% This is the standard estimate in~\cite[Proposition~4.14]{chen}.
We will prove the result by a standard application of the probabilistic method.
For all $i \in [M]$, choose each coordinate of each vector $u_i$ independently and uniformly from $\{-D^{-1/2},D^{-1/2}\}$.
Then, $u_i$ is a unit vector.
Note that for any $i \neq j$, $\ip{u_i}{u_j}$ is the sum of $D$ independent random variables with mean zero and bounded in magnitude by $1/D$.
Hoeffding's inequality (\cref{lem:hoeffding}) gives $\Pr[|\ip{u_i}{u_j}|>\xi]\le2e^{-D\xi^2/2}$.
A union bound over the $\binom{M}{2}$ pairs upper bounds the failure probability by
\[
    M^2 e^{-D\xi^2/2} \leq M^2 e^{-6 \log(2M) / 2} = M^2 (2M)^{-3} = \frac{1}{8M} < 1.
\]
Therefore, such a set of vectors must exist.
\end{proof}

\subsection{Proof of \cref{thm:lb-attn}}\label{sub:proof_lb}

Here we present the proof of the lower bound in the uniform sampling regime. Refer to the main body in \cref{sec:lower} for the construction of key, query, and value vectors.

\begin{proof}
Each group will contain $t=\lfloor n/m\rfloor\ge1$ copies of the
same key-value pair, with one additional copy in the first $s$
groups, so that $n=tm+s$ and $0\le s<m$.
For $i \in [m]$, let $N_i = t + 1$ if $i \leq s$ and $N_i = t$ if $i > s$.
The attention instance $K, V$ will have $N_i$ copies of the key-value pair $(k_i, v_i)$.
By the setting of $t, s, N_i$, there will be exactly $n$ key-value pairs.
As $t \geq 1$, $N_j/N_i \leq 2$ for all $i, j \in [m]$.

At distinct query $q_i$, let $\alpha_i$ be the attention mass on the copies of $(k_i,v_i)$. Bounding $\langle k_j, q_i\rangle - \langle k_i, q_i\rangle$ by \Cref{eq:matching_cosine} and \Cref{eq:non-matching_cosine}, we get
\begin{equation*}%\label{eq:mass}
 \alpha_i
 =\frac{N_i e^{\ip{k_i}{q_i}/\sqrt d}}{\sum_{j=1}^m N_j e^{\ip{k_j}{q_i}/\sqrt d}}
 =\frac{1}{1 + \sum_{j\neq i} \frac{N_j}{N_i} e^{(\ip{k_j}{q_i} - \ip{k_i}{q_i})/\sqrt d}}
 \geq \frac1{1+2(m-1)e^{-(1-2\xi)(b-a)R^2}}.
\end{equation*}

Recall the condition $m\le e^{(1-2\xi)(b-a)R^2}/32$.
Then, $1-\alpha_i\le2me^{-(1-2\xi)(b-a)R^2}\le\frac1{16}.$

The attention output assigns mass $\alpha_i$ to $v_i$ and total mass
$1-\alpha_i$ to the other value vectors. Since all value vectors have
unit norm, the triangle inequality gives
\begin{equation}\label{eq:lb-target-error-main}
    \|\Attn(K,V,q_i)-v_i\|
    \le (1-\alpha_i)\max_{j\ne i}\|v_j-v_i\|
    \le 2(1-\alpha_i)
    \le \frac18.
\end{equation}
Let $H_i$ be the event that the coreset retains at least one copy
of $(k_i,v_i)$. On $H_i^c$, the nonnegative coreset weights imply that
\[
    \widehat\Attn(K,V,q_i)=\sum_{j\ne i}\lambda_jv_j,
    \qquad
    \lambda_j\ge0,
    \qquad
    \sum_{j\ne i}\lambda_j=1.
\]
Since $|\ip{v_i}{v_j}|\le1/8$ by \Cref{eq:lb-value-inner-products} for every $j\ne i$, we obtain
\[
    \ip{v_i}{\widehat\Attn(K,V,q_i)}
    =\sum_{j\ne i}\lambda_j\ip{v_i}{v_j}
    \le\frac18\sum_{j\ne i}\lambda_j
    =\frac18.
\]

Since $\|v_i\|=1$, the Cauchy-Schwarz inequality
and~\cref{eq:lb-target-error-main} give
\[
    \ip{v_i}{\Attn(K,V,q_i)}
    =1+\ip{v_i}{\Attn(K,V,q_i)-v_i}
    \ge 1-\|\Attn(K,V,q_i)-v_i\|
    \ge \frac78.
\]
On $H_i^c$, we have already shown that
$\ip{v_i}{\widehat\Attn(K,V,q_i)}\le1/8$. Since $v_i$ is a unit vector, another
application of the Cauchy-Schwarz inequality gives
\begin{align*}
    \|\widehat\Attn(K,V,q_i)-\Attn(K,V,q_i)\|
    &\ge \ip{v_i}{\Attn(K,V,q_i)}
      -\ip{v_i}{\widehat\Attn(K,V,q_i)}\ge \frac78-\frac18
    =\frac34>\eps.
\end{align*}
Hence success at the fixed query $q_i$ implies $H_i$. Since answering $q_i$ accurately requires retaining a copy of
$(k_i,v_i)$, the success guarantee \Cref{eq:attention-guarantee} implies $\Pr[H_i]\ge2/3$
for every $i\in[m]$.
Since each retained pair represents at most one group,
\[
 \E|I|\ge\sum_{i=1}^m\Pr[H_i]\ge\frac{2m}{3}.
\]
By definition, $\OPTattn$ is the optimal worst-case number of tokens retained by an algorithm, so, from above, $\OPTattn \geq 2m/3$.
\end{proof}

\section{Proofs from \Cref{sec:upper_importance}}\label{sec:up_bound_importance_appendix}

\samplingguarantee*

\begin{proof}
If the algorithm retains all tokens, the estimate is exact.
Otherwise, fix $q\in\mathcal Q$ and write
\[
    Z=\sum_{i=1}^n e^{\ip{k_i}{q}/\sqrt d},
    \qquad
    Y=\sum_{i=1}^n e^{\ip{k_i}{q}/\sqrt d}v_i.
\]
Thus $\Attn(K,V,q)=Y/Z$. The sampled denominator and numerator
in \eqref{eq:upper-attention-estimator} are
\[
    \widehat Z
    =\frac1m\sum_{\ell=1}^m
        \frac{e^{\ip{k_{J_\ell}}{q}/\sqrt d}}{\pi_{J_\ell}},
    \qquad
    \widehat Y
    =\frac1m\sum_{\ell=1}^m
        \frac{e^{\ip{k_{J_\ell}}{q}/\sqrt d}v_{J_\ell}}
             {\pi_{J_\ell}}.
\]
Since each index $i$ is sampled with probability $\pi_i$,
we have $\E\widehat Z=Z$ and $\E\widehat Y=Y$.

The sensitivity upper bound gives
\begin{equation}\label{eq:upper-score-domination}
    e^{\ip{k_i}{q}/\sqrt d}\le\sigma_i Z\le s_i Z
    \qquad\text{for every }i\in[n].
\end{equation}
Using $\pi_i=s_i/S$,
\[
    \Var\widehat Z
    \le\frac1m\sum_{i=1}^n
        \frac{e^{2\ip{k_i}{q}/\sqrt d}}{\pi_i}
    \le\frac{SZ}{m}\sum_{i=1}^n e^{\ip{k_i}{q}/\sqrt d}
    =\frac{S}{m}Z^2.
\]
Similarly, using independence and applying \cref{lem:weighted-variance} to a centered sample, we obtain
\[
    \E\|\widehat Y-Y\|^2
    =\frac1m\left(
        \sum_{i=1}^n
            \frac{e^{2\ip{k_i}{q}/\sqrt d}\|v_i\|^2}{\pi_i}
        -\|Y\|^2
    \right)
    \le\frac1m\sum_{i=1}^n
        \frac{e^{2\ip{k_i}{q}/\sqrt d}}{\pi_i}
    \le\frac{S}{m}Z^2,
\]
where we used $\|v_i\|\le1$ and
\eqref{eq:upper-score-domination}.

Using Chebyshev's inequality (\cref{lem:chebyshev}) for $\widehat Z$ and Markov's inequality (\cref{lem:markov})
for $\|\widehat Y-Y\|^2$, followed by a union bound, we obtain
\[
    |\widehat Z-Z|\le\frac{\eps}{4}Z,
    \qquad
    \|\widehat Y-Y\|\le\frac{\eps}{4}Z
\]
hold together with probability at least $1-\frac{32S}{m\eps^2}\ge\frac23,$
since $m=\lceil96S/\eps^2\rceil$.

On this event, $\widehat Z\ge(1-\eps/4)Z$.
Also, $\|Y\|\le Z$ because $\|v_i\|\le1$. Therefore,
\[
    \left\|\frac{\widehat Y}{\widehat Z}-\frac{Y}{Z}\right\|
    \le\frac{\|\widehat Y-Y\|}{\widehat Z}
       +\frac{\|Y\|\,|\widehat Z-Z|}{Z\widehat Z}
    \le\frac{\eps/2}{1-\eps/4}
    \le\eps,
\]
where the last inequality uses $0<\eps<1$.
The algorithm retains at most $\min\{n,m\}$ tokens,
which gives the stated size bound.
\end{proof}

\totalsens*

\begin{proof}
In this proof, we replace each key
$k_i$ by $\sqrt{r_q/r_k}\,k_i$ and each query $q$ by
$\sqrt{r_k/r_q}\,q$. This preserves all inner products and hence
all sensitivities, while giving both keys and queries the
common norm bound $\sqrt{r_kr_q}=R\,d^{1/4}$.
We use the same notation for the rescaled keys and query set.

Since $\mathcal Q$ is compact and the attention weights are continuous, for
each $i\in[n]$ there exists a query $q_i \in \mathcal Q$ attaining the
sensitivity of token $i$. We let $q_i$ be a maximizing query for token $i$; these $q_1,\ldots,q_n$ need not be the same.

Define
\[\lambda_i=\sigma_i/\sum_{j=1}^n\sigma_j, 
    \qquad
    \bar k=\sum_{i=1}^n\lambda_i k_i,
    \qquad
    \bar q=\sum_{i=1}^n\lambda_i q_i.
\]
For each $i$, applying Jensen's inequality
(\cref{lem:jensen-log}, with each 
$x_j$ set to $e^{\ip{k_j}{q_i}/\sqrt d}$) gives
\begin{equation}\label{eq:jensen}
    \log\left(\sum_{j=1}^n e^{\ip{k_j}{q_i}/\sqrt d}\right)
    =\log\left(\sum_{j=1}^n\lambda_j
        \frac{e^{\ip{k_j}{q_i}/\sqrt d}}{\lambda_j}\right)
    \ge\sum_{j=1}^n\lambda_j
        \log\frac{e^{\ip{k_j}{q_i}/\sqrt d}}{\lambda_j}
    =\frac{\ip{\bar k}{q_i}}{\sqrt d}
        -\sum_{j=1}^n\lambda_j\log\lambda_j.
\end{equation}
By the definition of $q_i$,
$\sigma_i=e^{\ip{k_i}{q_i}/\sqrt d}/
\sum_{j=1}^n e^{\ip{k_j}{q_i}/\sqrt d}$.
Combining this with
$\sigma_i/\lambda_i=\sum_{j=1}^n\sigma_j$, we have
\[
    \log\left(\sum_{i=1}^n\sigma_i\right)
    =\sum_{i=1}^n\lambda_i
      \left(
        \frac{\ip{k_i}{q_i}}{\sqrt d}
        -\log\left(\sum_{j=1}^n e^{\ip{k_j}{q_i}/\sqrt d}\right)
        -\log\lambda_i
      \right).
\]
Applying \Cref{eq:jensen} to the right-hand side,
\begin{equation}\label{eq:covariance}
    \log\left(\sum_{i=1}^n\sigma_i\right)
    \le\sum_{i=1}^n\lambda_i\frac{\ip{k_i}{q_i}}{\sqrt d}
        -\frac{\ip{\bar k}{\bar q}}{\sqrt d}
    =\sum_{i=1}^n\lambda_i
        \frac{\ip{k_i-\bar k}{q_i-\bar q}}{\sqrt d}.
\end{equation}

We bound this expression in two ways. First,
$\sum_i\lambda_i\ip{k_i}{q_i}/\sqrt d\le h$. Also,
$\ip{\bar k}{\bar q}/\sqrt d
=\sum_{i,j}\lambda_i\lambda_j\ip{k_i}{q_j}/\sqrt d
\ge\ell$, since every $q_j$ belongs to $\mathcal Q$. Hence
\begin{equation}\label{eq:upper-mean-bound}
    \log\left(\sum_{i=1}^n\sigma_i\right)
    \le h-\frac{\ip{\bar k}{\bar q}}{\sqrt d}
    \le h-\ell.
\end{equation}

Second, centering the vectors (\cref{lem:weighted-variance}) gives $\sum_i\lambda_i\|k_i-\bar k\|^2
    =\sum_i\lambda_i\|k_i\|^2-\|\bar k\|^2
    \le R^2\sqrt d-\|\bar k\|^2,$
and the same bound holds for the queries.
Applying Cauchy-Schwarz (\cref{lem:weighted-cauchy-schwarz}) to \eqref{eq:covariance}, we obtain
\begin{align*}
    \log\left(\sum_{i=1}^n\sigma_i\right)
    &\le\frac{1}{\sqrt d}\sqrt{
        \left(\sum_i\lambda_i\|k_i-\bar k\|^2\right)
        \left(\sum_i\lambda_i\|q_i-\bar q\|^2\right)}\\ 
    &\le\frac{1}{\sqrt d}\sqrt{
        (R^2\sqrt d-\|\bar k\|^2)
        (R^2\sqrt d-\|\bar q\|^2)}\le R^2-\frac{\|\bar k\|^2+\|\bar q\|^2}{2\sqrt d}
    \le R^2+\frac{\ip{\bar k}{\bar q}}{\sqrt d}.
\end{align*}
The last two inequalities follow from the arithmetic-geometric mean
inequality (\cref{lem:am-gm}) and Young's inequality for inner products,
$2\ip{\bar k}{\bar q}\ge -\|\bar k\|^2-\|\bar q\|^2$, respectively.
Averaging this bound with
$\log(\sum_{i=1}^n\sigma_i)
\le h-\ip{\bar k}{\bar q}/\sqrt d$
from \eqref{eq:upper-mean-bound} gives
\[
    \log\left(\sum_{i=1}^n\sigma_i\right)
    \le\frac{R^2+h}{2}.
\]
Together with
$\log(\sum_{i=1}^n\sigma_i)\le h-\ell$,
this proves the lemma.

\end{proof}
% \begin{proof}[Proof of \Cref{lem:uniform-sampling}]
% Choose $s_i=e^{(b-a)R^2}/n$ for every $i\in[n]$.
% By \eqref{eq:bound-kernel}, these scores satisfy $s_i\ge\sigma_i$.
% They give $\pi_i=1/n$ and $S=e^{(b-a)R^2}$, so
% \cref{lem:upper-sampling} yields the claimed sample bound
% and approximation guarantee. Each sampled occurrence has
% the same weight $n/m$, which cancels in the attention ratio.
% \end{proof}

\computescores*

\begin{proof} 

We group keys and queries by norm, compute sampling scores
within each pair of groups by convex optimization, and then
bound the sum of these scores using \cref{lem:total-sensitivity}.
\paragraph{Partitioning by norm.}
Set $M=\max\{1,\lceil R^2\rceil\}$.
Partition the keys into norm bins
\[
    I_1=\left\{i\in[n]:\|k_i\|\le\frac{r_k}{M}\right\},
    \qquad
    I_u=\left\{i\in[n]:
        \frac{(u-1)r_k}{M}<\|k_i\|\le\frac{ur_k}{M}
    \right\}
    \quad (u=2,\ldots,M).
\]
Similarly, partition the possible query norms into $M$ intervals
of width $r_q/M$, assigning each endpoint to the interval on
its left and including zero in the first interval.

For $u,v\in[M]$, denote
\[
    A_{uv}=\frac{R^2(u-1)(v-1)}{M^2},
    \quad
    B_{uv}=\frac{R^2uv}{M^2},
    \quad
    \ell_{uv}=\min\{aA_{uv},aB_{uv}\},
    \quad
    h_{uv}=\max\{bA_{uv},bB_{uv}\}.
\]
If $i\in I_u$ and an admissible query $q$ has norm in the
$v$-th interval, then
\[
    A_{uv}\le\frac{\|k_i\|\|q\|}{\sqrt d}\le B_{uv},
    \qquad
    \ell_{uv}\le\frac{\ip{k_i}{q}}{\sqrt d}\le h_{uv}.
\]
Thus, for each nonempty $I_u$, the compact convex domain
\[
    \widetilde{\mathcal Q}_{uv}
    =\left\{z\in\R^d:
        \|z\|\le\frac{vr_q}{M},\quad
        \ell_{uv}\le\frac{\ip{k_j}{z}}{\sqrt d}\le h_{uv}
        \quad\forall j\in I_u
    \right\}
\]
contains every admissible query whose norm lies in the
$v$-th interval.

\paragraph{Bounding sensitivities locally.}
For $i\in I_u$, define the local sensitivity
\[
    \sigma_i^{uv}
    =\max_{z\in\widetilde{\mathcal Q}_{uv}}
        \frac{e^{\ip{k_i}{z}/\sqrt d}}
             {\sum_{j\in I_u}e^{\ip{k_j}{z}/\sqrt d}},
\]
with $\sigma_i^{uv}=0$ when
$\widetilde{\mathcal Q}_{uv}$ is empty.
For an admissible query $q$ whose norm lies in the $v$-th
interval,
\[
    \frac{e^{\ip{k_i}{q}/\sqrt d}}
         {\sum_{j=1}^n e^{\ip{k_j}{q}/\sqrt d}}
    \le
    \frac{e^{\ip{k_i}{q}/\sqrt d}}
         {\sum_{j\in I_u}e^{\ip{k_j}{q}/\sqrt d}}
    \le\sigma_i^{uv}.
\]
Consequently, $\sigma_i\le\sum_{v=1}^M\sigma_i^{uv}$.

\paragraph{Computing the sampling scores.}
For $i\in I_u$, define
\[
    f_i^u(z):=\log\!\left(
        \sum_{j\in I_u}e^{\ip{k_j-k_i}{z}/\sqrt d}
    \right).
\]
By \cref{lem:log-sum-exp}, this objective is convex and continuously
differentiable, with an explicit gradient. Whenever
$\widetilde Q_{uv}$ is nonempty, it is the logarithm of the reciprocal
of token $i$'s attention weight within $I_u$, so
\[
    \min_{z\in\widetilde Q_{uv}}f_i^u(z)
    =-\log\sigma_i^{uv}.
\]
Moreover, for every $z\in\widetilde Q_{uv}$, $0\le f_i^u(z)\le\log|I_u|+h_{uv}-\ell_{uv}.$
The lower bound follows from the term $j=i$, and the upper bound
follows from the defining inner-product bounds of
$\widetilde Q_{uv}$.

By \cref{lem:norm-affine-constraints} and the definition of
$\widetilde Q_{uv}$, each nonempty domain is compact and convex.
Applying \cref{lem:additive-convex-minimization} with
$\eta=\log 2$, we  obtain
$z_i^{uv}\in\widetilde Q_{uv}$ satisfying
\begin{equation}\label{eq:log2-bound}
        -\log\sigma_i^{uv}
    \le f_i^u(z_i^{uv})
    \le-\log\sigma_i^{uv}+\log2.
\end{equation}

Set $s_i^{uv}:=2e^{-f_i^u(z_i^{uv})}$.
Exponentiating \cref{eq:log2-bound} gives $\sigma_i^{uv}\le s_i^{uv}\le2\sigma_i^{uv}.$
For empty domains, set $s_i^{uv}:=0$, so the same inequalities
hold under the convention $\sigma_i^{uv}=0$. 
We are now ready to describe the sampling scores $s_i$. 
For $i\in I_u$,
define
\[
    s_i:=\sum_{v=1}^M s_i^{uv}.
\]
Together with the local sensitivity bound, this gives $\sigma_i
    \le\sum_{v=1}^M\sigma_i^{uv}
    \le s_i
    \le2\sum_{v=1}^M\sigma_i^{uv}.$
Since $0\in Q$, we have $\sigma_i\ge1/n$, so the scores are positive.

\paragraph{Bounding the sum within each pair of bins.}
Write $D_{uv}=B_{uv}-A_{uv}
    =\frac{R^2(u+v-1)}{M^2}
    \le\frac{2R^2}{M}\le2.$
For $x_+=\max\{x,0\}$, the endpoint definitions give
\[
    \ell_{uv}=aB_{uv}-a_+D_{uv},
    \qquad
    h_{uv}=bB_{uv}+(-b)_+D_{uv}.
\]
Since $a_++(-b)_+\le1$,
\[
    h_{uv}-\ell_{uv}\le(b-a)B_{uv}+2,
    \qquad
    \frac{B_{uv}+h_{uv}}2
    \le\frac{1+b}{2}B_{uv}+1.
\]

The keys in $I_u$ have norm at most $ur_k/M$, and queries
in $\widetilde{\mathcal Q}_{uv}$ have norm at most $vr_q/M$.
The product of these bounds, divided by $\sqrt d$, is $B_{uv}$.
For each nonempty domain, \cref{lem:total-sensitivity} therefore gives
\[
    \sum_{i\in I_u}\sigma_i^{uv}
    \le\exp\left(
        \min\left\{
            h_{uv}-\ell_{uv},
            \frac{B_{uv}+h_{uv}}2
        \right\}
    \right)
    \le e^{\kappa B_{uv}+2}.
\]
The same bound holds trivially for empty domains. Hence
\[
    \sum_{i=1}^n s_i
    \le2e^2\sum_{u=1}^M\sum_{v=1}^M e^{\kappa B_{uv}}.
\]

\paragraph{Summing over the bins.}
If $R^2<1$, then $M=1$, and the desired bound follows
immediately. Otherwise, $M=\lceil R^2\rceil\le2R^2$.
Since $u,v\le M$,
\[
    B_{uv}
    \le\frac{R^2(u+v)}{2M}
    =R^2-\frac{R^2}{2M}\bigl((M-u)+(M-v)\bigr).
\]
Using $R^2/(2M)\ge1/4$, we obtain
\[
    \sum_{u=1}^M\sum_{v=1}^M e^{\kappa B_{uv}}
    \le e^{\kappa R^2}
        \left(\sum_{t=0}^{M-1}e^{-\kappa t/4}\right)^2
    \le\frac{e^{\kappa R^2}}{(1-e^{-\kappa/4})^2}.
\]
Finally, $0<\kappa\le1$ and
$1-e^{-\kappa/4}\ge\kappa/8$, so
\[
    \sum_{i=1}^n s_i
    \le\frac{128e^2}{\kappa^2}e^{\kappa R^2}.
\]
Each key participates in at most $M$ optimization problems,
giving at most $nM$ such problems in total.
\end{proof}

\section{Proofs from \cref{sec:importance_lower}}\label{sec:importance_lower_appendix}

In this section, we give the proof of \cref{thm:lb-attn-full}, the lower bound in the importance sampling regime. \begin{theorem}[Lower bound, \Cref{item:importance}]\label{thm:lb-attn-full}
Consider $r_k,r_q>0$, $\eps\in(0,1/2]$, and $-1\le a<b\le1$ in the importance sampling regime (\cref{def:small_cos_sim}).
Consider $n$ and $d=\omega(\log n)$, and let $\xi=\sqrt{\frac{6\log(2n)}{d-3}}=o(1)$.
Pick any integer $m$ satisfying
\begin{equation}\label{eq:finite-conditions-appendix}
    2\le m\le n,\qquad m\le \min\left\{
    \mainterm{e^{\frac{1+b}{2}R^2}},
    \mainterm{e^{(b-a)R^2}}
    \right\}\secondary{\cdot e^{-2\xi R^2}/32}.
\end{equation}
There is a fixed dataset $(K,V)$ of $n$ key-value pairs, with keys of norm $r_k$ and unit-norm values in $\R^d$, together with queries $q_1,\ldots,q_m$ of norm $r_q$, such that every key $k\in K$ satisfies
\[
    a\le\frac{\ip{k}{q_i}}{\|k\|\|q_i\|}\le b
    \qquad\text{for all }i\in[m].
\]
For this dataset, satisfying~\Cref{eq:attention-guarantee} at each
of these queries requires $\OPTattn\ge2m/3$.
\end{theorem}
The construction and analysis are essentially the same as that in the uniform sampling regime \cref{sec:lower} with slightly more complicated parameter settings in order to capture all possible values of $a,b$.

\paragraph{Goal.}
We construct keys and queries such that each matching pair $(k_i,q_i)$
has scaled inner product $bR^2$, exceeding that of every nonmatching
pair $(k_j,q_i)$, $j\ne i$, by at least
$(1-o(1))\min\{b-a,(1+b)/2\}R^2$, where $R^2=r_kr_q/\sqrt d$.
We will then use this gap to make each query concentrate
its attention on copies of its matching key, forcing an accurate
summary to retain many key groups.

\paragraph{Construction.}
Let $e_0,e_K,e_Q$ be the first three standard basis vectors.
Choose $m$ unit vectors satisfying \cref{lem:code} in dimension
$D=d-3$, with pairwise inner products bounded in magnitude by $\xi = \sqrt{\frac{6\log(2n)}{D}}$.
Embed them into the last $d-3$ coordinates of $\R^d$, and denote the
resulting vectors by $u_1,\ldots,u_m$. Define 
\begin{equation}\label{eq:lb-geometry-appendix}%\label{eq:geometry}
\begin{aligned}
    k_i&=r_k\bigl(\underbrace{\sqrt{\alpha}\,u_i}_{\text{Shared group direction}}+\underbrace{\sqrt{|\beta|}\,e_0}_{\text{Anisotropic shift}} + \quad \underbrace{\sqrt{1-\alpha-|\beta|}\,e_K}_{\text{Norm padding}}\; \bigr), \\
    q_i&=r_q\bigl(\underbrace{\sqrt{\alpha}\,u_i}_{\text{Shared group direction}}- \underbrace{\operatorname{sgn}(\beta) \sqrt{|\beta|}\,e_0}_{\text{Anisotropic shift}} \quad + \quad \underbrace{\sqrt{1-\alpha-|\beta|}\,e_Q}_{\text{Norm padding}}\; \bigr),
    % q_i&=r_q\bigl(\sqrt{\alpha}\,u_i
      % -\sqrt{|\beta|}\,e_0\bigr)+ r_q\sqrt{1-\alpha-|\beta|}\,e_Q,
\end{aligned}
\end{equation}
where we set
\begin{equation}\label{eq:geometry-parameters-main-appendix}
    \alpha=\min\left\{\frac{1+b}{2},\frac{b-a}{1+\xi}\right\},
    \qquad
    \beta=\alpha - b = \min\left\{\frac{1-b}{2}, \frac{-a-\xi b}{1+\xi}\right\}.
    % h=bR,\qquad
    % L=\min\left\{\frac{R+h}{2},\frac{(b-a)R}{1+\xi}\right\},
    % \qquad \beta=L-h.
\end{equation}
Note that $\alpha \geq 0$ as $b \geq -1$ and $b \geq a$.
Note that $\alpha + |\beta| \leq 1$ as $\alpha + \beta = 2\alpha - b \leq 2\left(\frac{1 + b}{2}\right) - b = 1$ and $\alpha - \beta = b \leq 1$.
Therefore, $1 - \alpha - |\beta| \geq 0$ and the keys and queries are well defined with norms exactly equal to $\|k_i\| = r_k$ and $\|q_i\| =r_q$.
The components along $u_i$ and $e_0$ serve the purpose of achieving the desired inner products while the padding along $e_K$ and $e_Q$ increases the norms of the keys/queries to fix the denominator of cosine similarity.

\paragraph{Verifying the inner products.}

It remains to bound the scaled inner products between queries and keys.
For any $i \in [m]$, the inner product between the matching pair is:
\[
    \frac{\ip{k_i}{q_i}}{\sqrt d} = \left(\alpha - \beta\right)R^2 = bR^2.
\]
For any non-matching pair $i \neq j$,
\[
    \left(b-\alpha\left(1+\xi\right)\right)R^2
    \le \frac{\ip{k_i}{q_j}}{\sqrt d}
    =\left(\alpha \ip{u_i}{u_j}-\beta\right)R^2\le bR^2.
\]
The second choice in the minimum defining $\alpha$ ensures $b-\alpha\left(1+\xi\right)\ge a.$
Thus, every key-query cosine similarity belongs to $[a,b]$, since
$R^2\sqrt d=r_kr_q=\|k_i\|\|q_j\|$.

For any $i \neq j$, the scaled inner product between key-query pairs at the same index exceeds that of any other key-query pair by
\[
\frac{\ip{k_i}{q_i} - \ip{k_j}{q_i}}{\sqrt d}
\geq bR^2-\left(\alpha\xi-\beta\right)R^2
=\alpha\left(1-\xi\right)R^2.
\]
By the definition of $\alpha$,
\begin{equation}\label{eq:lb-kv-gap}
\frac{\ip{k_i}{q_i} - \ip{k_j}{q_i}}{\sqrt d}
\geq \min\left\{b-a,\frac{1+b}{2}\right\}R^2\frac{1-\xi}{1+\xi}
\ge (1 - 2\xi)\min\left\{b-a,\frac{1+b}{2}\right\}R^2.
\end{equation}
The last inequality uses $(1-x)/(1+x)\ge1-2x$ for $x\ge0$.

\paragraph{Constructing the value vectors.}
Apply \cref{lem:code} a second time to obtain $m$ unit value vectors
$v_1,\ldots,v_m\in\R^d$ satisfying
\begin{equation}\label{eq:lb-value-inner-products-appendix}
    |\ip{v_i}{v_j}|\le\frac18
    \qquad\text{for all }i,j\in[m]\text{ with }i\ne j.
\end{equation}
Since $m\le n$, the assumption $d=\omega(\log n)$ ensures that this
is possible.

\subsection{Full construction}
\begin{proof}[Proof of \cref{thm:lb-attn-full}]
Choose $t=\lfloor n/m\rfloor\ge1$ and
$0\le s<m$ such that $n=tm+s$.
For $i \in [m]$, let $N_i = t + 1$ if $i \leq s$ and $N_i = t$ if $i > s$.
The attention instance $K, V$ will have $N_i$ copies of the key-value pair $(k_i, v_i)$.
By the setting of $t, s, N_i$, there will be exactly $n$ key-value pairs.
As $t \geq 1$, $N_j/N_i \leq 2$ for all $i, j \in [m]$.

At distinct query $q_i$, let $\alpha_i$ be the attention mass on the copies of $(k_i,v_i)$.
By~\cref{eq:lb-kv-gap},
\begin{equation*}%\label{eq:mass}
 \alpha_i
 =\frac{N_i e^{\ip{k_i}{q_i}/\sqrt d}}{\sum_{j=1}^m N_j e^{\ip{k_j}{q_i}/\sqrt d}}
 =\frac{1}{1 + \sum_{j\neq i} \frac{N_j}{N_i} e^{(\ip{k_j}{q_i} - \ip{k_i}{q_i})/\sqrt d}}
 \geq \frac1{1+2(m-1)e^{-(1-2\xi)\min\{b-a,(1+b)/2\}R^2}}.
\end{equation*}

Recall the condition $m\le e^{(1-2\xi)\min\{b-a,(1+b)/2\}R^2}/32$.
Then,
\[
 1-\alpha_i\le2me^{-(1-2\xi)\min\{b-a,(1+b)/2\}R^2}\le\frac1{16}.
\]

So, attention focuses significant mass on $v_i$:
\begin{equation}\label{eq:lb-target-error}
 \|\Attn(K,V,q_i)-v_i\|
 \le2(1-\alpha_i)\le\frac18.
\end{equation}
The first inequality follows because all value vectors have norm one, so their distance from $v_i$ is at most two.

Let $H_i$ be the event that the coreset retains at least one copy
of $(k_i,v_i)$.
On the complementary event, the approximation at $q_i$ is a convex combination of the value vectors $v_j$ with $j\ne i$. This follows
directly from the nonnegative weights in the definition of query-oblivious weighted token-eviction. Therefore
\[
 \ip{v_i}{\widehat\Attn(K,V,q_i)}\le \frac18
\]
as all distinct value vectors are nearly orthogonal.

On the other hand,~\cref{eq:lb-target-error} gives
$\ip{v_i}{\Attn(K,V,q_i)}\ge1-1/8$.
Projecting the error onto the unit vector $v_i$ yields, whenever
$H_i$ does not hold,
\begin{equation*}
\begin{aligned}
 \|\widehat\Attn(K,V,q_i)-\Attn(K,V,q_i)\|
 &\ge |\ip{v_i}{\widehat\Attn(K,V,q_i)}-\ip{v_i}{\Attn(K,V,q_i)}|\\
 &\ge1-\frac18-\frac18=\frac34>\eps.
\end{aligned}
\end{equation*}
Hence success at the fixed query $q_i$ implies $H_i$.

The per-query guarantee of a query-oblivious weighted token-eviction policy now gives $\Pr[H_i]\ge2/3$ for each $i$ separately.
Since each retained pair represents at most one group,
\[
 \E|I|\ge\sum_{i=1}^m\Pr[H_i]\ge\frac{2m}{3}.
\]
By definition, $\OPTattn$ is the optimal worst-case number of tokens retained by an algorithm, so, from above, $\OPTattn \geq 2m/3$.
\end{proof}

\newpage
\section{Experiments}

\subsection{Q/K Cosine similarity}
\label{sec:qk_experiments}
\begin{figure}[H]
    \centering
    \includegraphics[width=\linewidth]{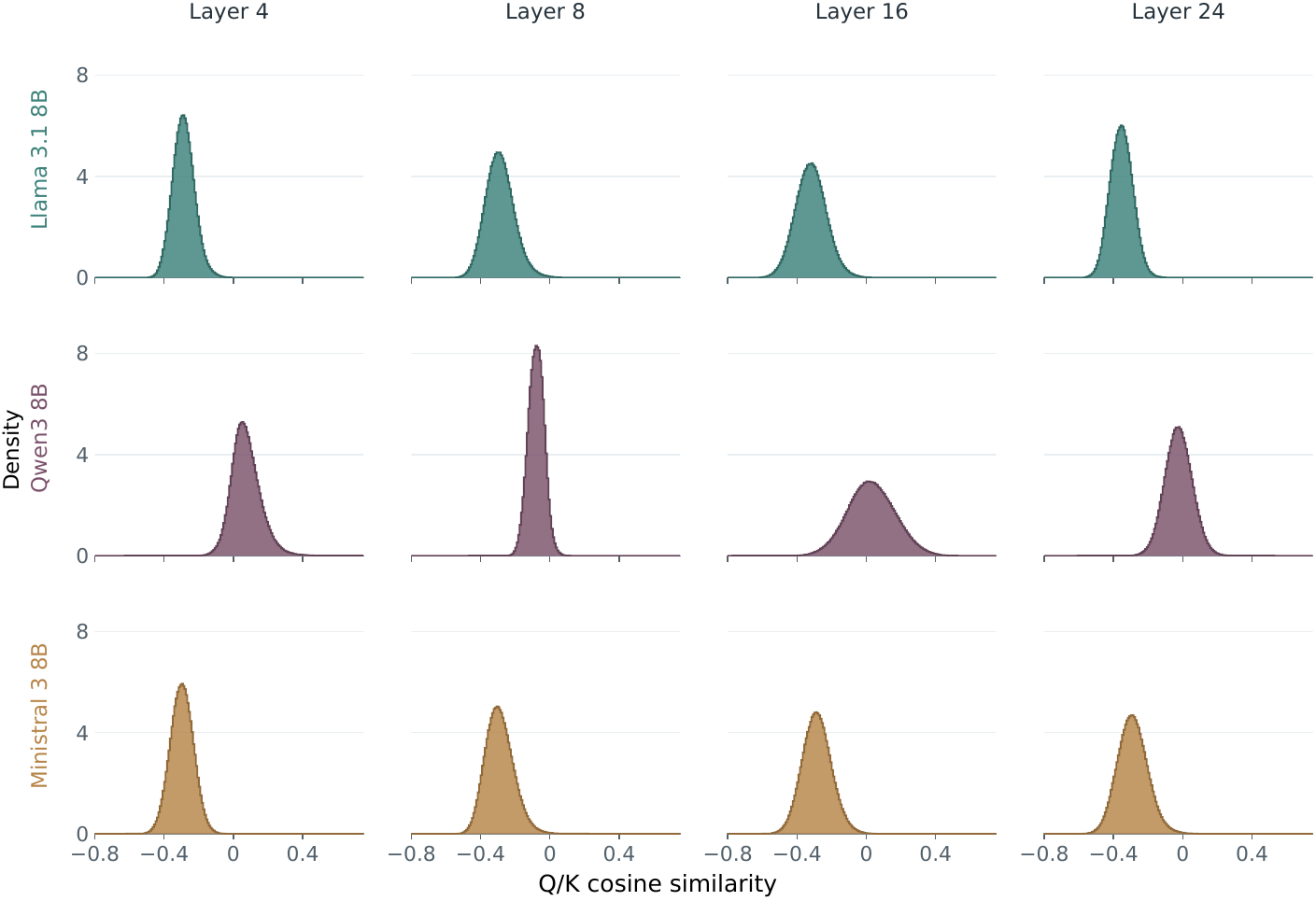}
    \caption{\textbf{Query–key cosine similarities across models and attention layers.} Rows show Llama 3.1 8B, Qwen3 8B, Ministral 3 8B; columns show layers 4, 8, 16, 24. All panels use the same prepared LongBench v2 example. Each panel uses an input of 8,192 tokens and compares the final 256 queries from query head 0 against 7,935 earlier-context keys from the corresponding KV head 0. The initial key  is excluded. Cosines are computed after model normalization and RoPE.}
    \label{fig:cos_sim_gridplot}
\end{figure}

\begin{figure}[H]
    \centering
    \includegraphics[width=0.7\linewidth]{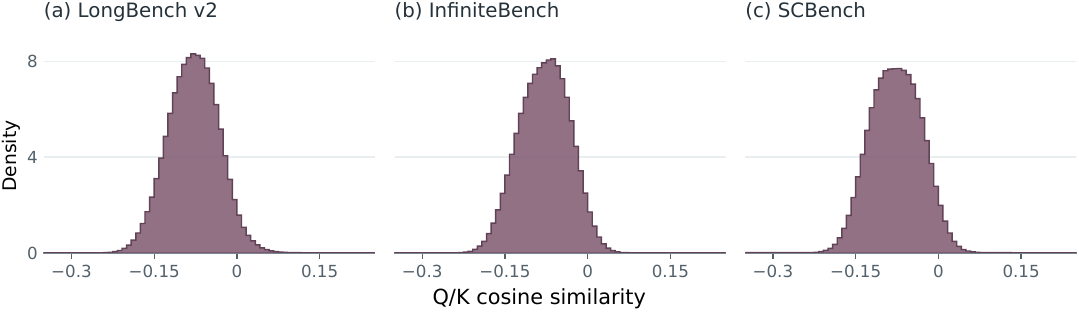}
    \caption{\textbf{Query–key cosine similarities across inputs.} The plot uses Qwen3 8B at layer 8 on (a) LongBench v2, (b) InfiniteBench, (c) SCBench. Each panel uses an input of 8,192 tokens and compares the final 256 queries from query head 0 against 7,935 earlier-context keys from the corresponding KV head 0. The initial key (attention sink) is excluded. Cosines are computed after model normalization and RoPE.}
    \label{fig:cos_sim_datasets}
\end{figure}

\begin{figure}[H]
    \centering
    \includegraphics[width=\linewidth]{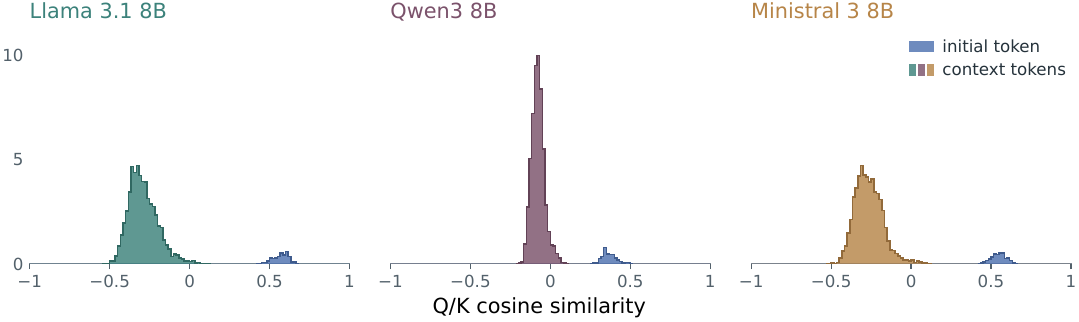}
    \caption{\textbf{Query–key cosine similarities with initial token (attention sink)} Qwen3 8B at layer 8 on LongBench v2 example input of 8,192 tokens and compares the final 256 queries from query head 0 against a sample of 16 earlier-context keys from the corresponding KV head 0 and the initial key (attention sink). Cosines are computed after model normalization and RoPE.}
    \label{fig:cos_sim_with_sink}
\end{figure}

\subsection{Key/Query norms}
\begin{figure}[H]
    \centering
    \includegraphics[width=\linewidth]{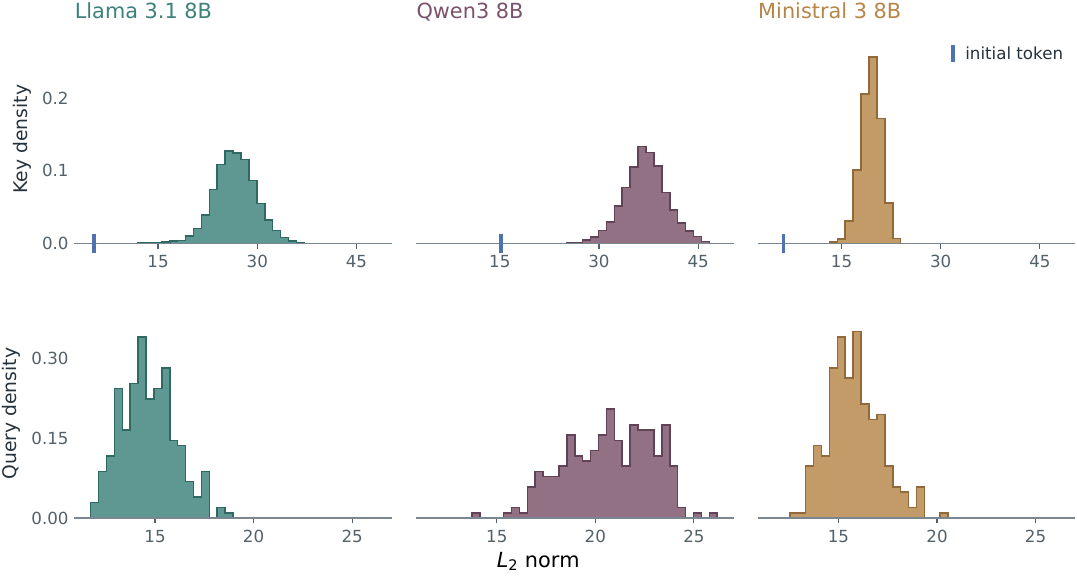}
    \caption{\textbf{Key and query magnitudes across models.} Layer 8 on the LongBench v2 input (8,192 tokens). Columns show Llama 3.1 8B, Qwen3 8B, Ministral 3 8B. The top row uses 7,935 keys per model: all eligible earlier-context keys. Blue baseline marks indicate the initial key (attention sink) norm in each model. The bottom row uses the final 256 queries.}
    \label{fig:norms}
\end{figure}

\begin{comment}
\subsection{Comparison with theory}
\begin{figure}[H]
    \centering
    \includegraphics[width=\linewidth]{img/delta_vs_delta_star_combined_all_datasets_q0.pdf}
    \caption{\textbf{Empirical cosine widths versus the reference crossover.} Each point represents a model, dataset, context length, and attention layer using prepared inputs from LongBench v2, InfiniteBench (KVPress), and SCBench (query head 0). Colors identify models, shapes distinguish layers. Datasets and context lengths are not visually distinguished. The empirical width is $\delta=[P_{99}(C)-P_1(C)]/2$, where $C$ pools signed Q/K cosine similarities within each setup after excluding BOS and query-tail keys. The reference $\delta^{*}$ satisfies $\kappa(\delta^{*})R=\ln n$, with $\kappa(\delta)=\min\{2\delta,(1+\delta)/2\}$, $R=(\operatorname{mean}_i\|k_i\|_2)(\operatorname{mean}_j\|q_j\|_2)/\sqrt{128}$, and $n$ the number of eligible context keys. The dashed diagonal marks equality; shading indicates $\delta<\delta^{*}$.}
    \label{fig:delta_vs_delta*}
\end{figure}
\end{comment}

\end{document}